%% file: neurips_2026.tex
\documentclass{article}

\usepackage[final,main]{neurips_2026}

\usepackage[utf8]{inputenc}
\usepackage[T1]{fontenc}
\usepackage{hyperref}
\usepackage{url}
\usepackage{booktabs}
\usepackage{amsfonts}
\usepackage{amsmath}
\usepackage{amssymb}
\usepackage{amsthm}
\usepackage{nicefrac}
\usepackage{microtype}
\usepackage[table]{xcolor}
\usepackage[normalem]{ulem}
\newif\ifrspchanges
\rspchangestrue
\definecolor{rspadded}{RGB}{0,0,0}
\newcommand{\rspadd}[1]{\ifrspchanges{\color{rspadded}#1}\else#1\fi}

\usepackage{graphicx}
\usepackage{multirow}
\usepackage{enumitem}
\usepackage{subcaption}
\usepackage{float}
\usepackage{wrapfig}
\usepackage[most]{tcolorbox}
\usepackage{fontawesome5}

\newsavebox{\rspPerformanceTable}
\newlength{\rspPerformanceTableWidth}

\newtheorem{theorem}{Theorem}
\newtheorem{proposition}{Proposition}

\newtheorem{corollary}{Corollary}

\title{From Noise to Diversity:\\ Random Embedding Injection in LLM Reasoning}

\author{%
  \textbf{Heejun Kim}\textsuperscript{1,\,*} \quad \textbf{Seungpil Lee}\textsuperscript{1,\,*} \quad \textbf{Jewon Yeom}\textsuperscript{2} \quad \textbf{Jaewon Sok}\textsuperscript{2} \\
  \textbf{Seonghyeon Park}\textsuperscript{2} \quad \textbf{Jeongjae Park}\textsuperscript{2} \quad \textbf{Taesup Kim}\textsuperscript{2,\,\dag} \quad \textbf{Sundong Kim}\textsuperscript{1,\,\dag} \\[1ex]
  \textsuperscript{1}Gwangju Institute of Science and Technology \quad \textsuperscript{2}Seoul National University \\[0.5ex]
}

\begin{document}

\maketitle
{\renewcommand{\thefootnote}{}\footnotetext{\textsuperscript{*}Equal contribution. \quad \textsuperscript{\dag}Co-corresponding authors.}}

\begin{abstract}
Recent \textit{soft prompt} research has tried to improve reasoning by inserting trained vectors into LLM inputs, yet whether the gain comes from the \emph{learned content} or from the \emph{act of injection itself} has not been carefully separated. We study Random Soft Prompts (RSPs), which drop the training step entirely and append a freshly drawn sequence of random embedding vectors to the input. Each RSP vector is sampled from an isotropic Gaussian fitted to the entrywise mean and variance of the pretrained embedding table. The sequence carries no learned content, yet reaches accuracy comparable to optimized \textit{soft prompts} on math reasoning benchmarks in several settings. The mechanism unfolds in two stages: because attention has to absorb a never-seen-before random position, the distribution over the first few generated tokens flattens and reasoning trajectories \emph{branch}, and as generation continues this influence dilutes naturally so the response commits to a single completion. We show that during inference RSPs lift early-stage token diversity and, combined with temperature sampling, widen Pass@$N$, the probability that at least one out of $N$ attempts is correct. \rspadd{Beyond inference, RSP collects correct solutions for more hard problems and improves subsequent supervised fine-tuning.} Our contributions are: (i) RSP isolates the simplest form of \textit{soft prompt}, which is training-free and freshly resampled, providing a unified lens for the structural effect of injection that variants otherwise differing in training and form all share, (ii) a theoretical and empirical validation of the underlying mechanism, and \rspadd{(iii) an application of RSP to training-data collection that improves subsequent supervised fine-tuning without RSP at inference.}
\end{abstract}
\begin{center}
\vspace{-0.4em}\faGlobe\ Project page: \href{https://heejunkim00.github.io/projects/RSP/}{\texttt{heejunkim00.github.io/projects/RSP}}
\end{center}

\input{sections/1_introduction}
\input{sections/2_related_work}
\input{sections/4_why_rsp_works}
\input{sections/3_experimental_analysis}
\input{sections/5_application}
\input{sections/6_conclusion}

\newpage
\section*{Acknowledgements}
We thank Laura Mascarell, Woochang Sim, Jumyung Park, and the anonymous reviewers for their valuable feedback. This work was supported by IITP (RS-2024-00436680, 40\%; RS-2024-00445087, 20\%; RS-2023-00216011, 10\%; RS-2026-25614465, 10\%; No. 2019-0-01842; 10\%), NRF (RS-2024-00345809, 10\%) funded by the Ministry of Science and ICT, Korea. Computing resource is supported by AICA, GIST SCENT-AI, and Advanced GPU Utilization Support Program (No. 02-26-01-0181).

\bibliographystyle{plainnat}
\bibliography{biblist}

\appendix
\input{appendix/appendix}

\end{document}

%% file: sections/1_introduction.tex
\section{Introduction}
\label{sec:introduction}

As LLMs scale to billions of parameters, full fine-tuning is prohibitive, motivating parameter-efficient methods \citep{hu2022lora, houlsby2019adapter}. Among these, \textit{soft prompts} insert learnable continuous embeddings into the input and steer behavior through attention \citep{li2021prefix, lester2021power}. This paradigm is actively adopted for mathematical reasoning \citep{kang2025model, xu2025softcot, ye2025thinking, hao2024training, zhang2025memgen}. These methods diverge in injection position, training objective, and prompt form, and prior work attributes their gains to the \emph{learned content} of the optimized vectors. The act of injecting embeddings into the input, common to all variants, has remained outside the analysis. Even though theory shows trained prefixes only bias attention in a fixed direction \citep{petrov2024prefix}, the source of the effect itself remains unsettled.

\begin{figure}[!t]
  \centering
  \includegraphics[width=\textwidth]{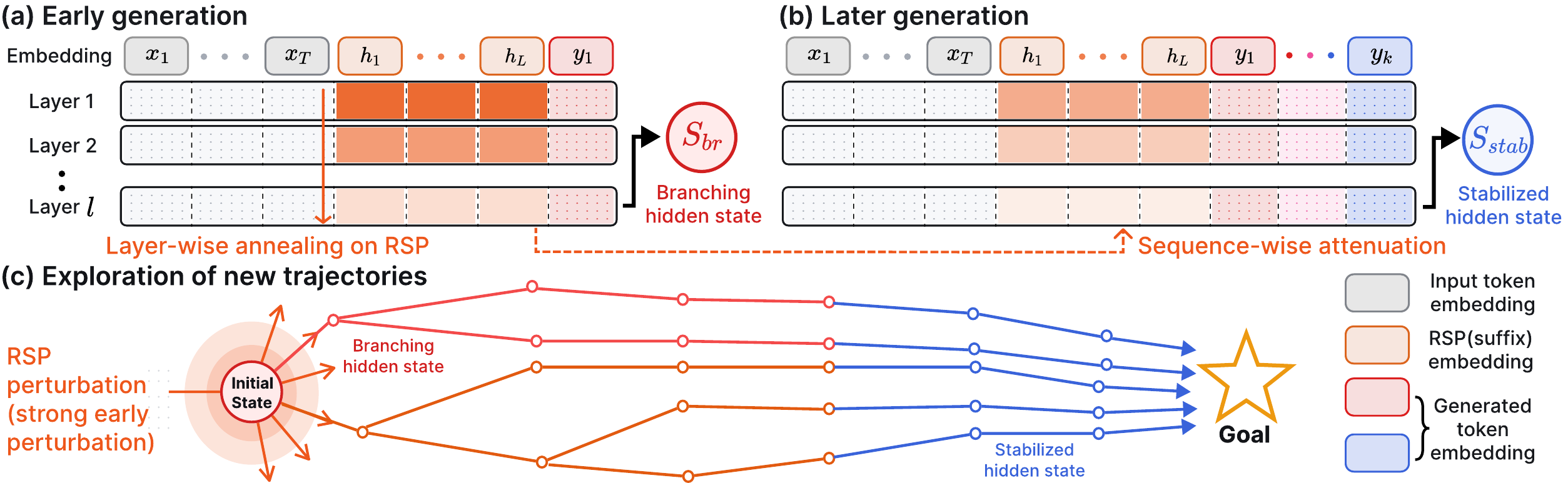}
  \caption{\rspadd{\textbf{Conceptual overview of RSP-induced trajectory diversity.}} The hidden states shown are the final-layer outputs that drive the next-token distribution. (a) Early decoding: RSP perturbation reaches a \emph{branching state} (red\rspadd{,} formalized in \S\ref{sec:rsp_definition}), whose top-$K$ differs from the no-RSP baseline and induces a diverse output distribution. (b) Later decoding: \rspadd{As generation proceeds, RSP attention mass decreases and the hidden state (blue) stabilizes along the branch opened earlier.} (c) Across rollouts, the trajectories opened by early RSP perturbation accumulate into independent paths.}
  \label{fig:main}
\end{figure}

To separate \emph{learned content} from the \emph{act of injection}, we use Random Soft Prompts (RSPs) as a control\rspadd{:} continuous vectors drawn from an isotropic Gaussian fitted to the embedding table's mean and variance, with no learning. The control turns out to be non-neutral. Applying a chat template to a base model not fine-tuned for it degrades reasoning, yet \textit{prefix} RSPs recover $+27.16$ and $+28.62$ pp on MATH-500 and GSM8K, respectively, for Qwen2.5-Math-1.5B (Table~\ref{tab:performance}). The injection itself, independent of any learned content, alters model behavior. \rspadd{To understand how injection alone produces these effects, we examine the resulting generation dynamics.} With each rollout receiving an \emph{independent} draw, a single RSP injection \emph{branches} hidden states into diverse trajectories early on, which then \emph{stabilize} as decoding proceeds. Empirically, RSPs share the early-decoding entropy signature of optimized methods (LTPO, TTSV, SoftCoT) despite carrying no learned content. \rspadd{Independent RSP resampling improves Pass@$N$. RSP also collects correct solutions for more hard problems, improving subsequent supervised fine-tuning without RSP at inference.}

Our contributions are: (i) RSP isolates the simplest form of \textit{soft prompt}\rspadd{,} training-free and freshly resampled per rollout\rspadd{,} providing a unified lens for the structural effect of injection across variants that otherwise differ in training objective and prompt form, while still reaching accuracy comparable to optimized methods in several settings\rspadd{.} (ii) \rspadd{We provide} a theoretical and empirical account of the mechanism: attention-mediated branch opening, isotropic coverage across all directions, and \rspadd{the observed attenuation of RSP attention explained through token counts and the evolving attention-logit gap.} \rspadd{(iii) We apply RSP to rejection-sampling SFT, showing that correct solutions collected with RSP improve downstream accuracy without RSP at inference.}

%% file: sections/2_related_work.tex

\section{Background}\label{sec:background}

\subsection{Soft Prompt Methods for LLM Reasoning}\label{sec:rw_soft_prompt}

Prior \textit{soft prompt} methods share a common structure: they inject continuous vectors that do not correspond to discrete vocabulary tokens into the model's representation stream and \emph{optimize} them under task-specific objectives. As parameter-efficient alternatives to full fine-tuning, Prefix-Tuning \citep{li2021prefix}, Prompt Tuning \citep{lester2021power}, and P-tuning \citep{liu2021gpt} reached near-fine-tuning performance using learnable continuous vectors, after which LLM-reasoning variants followed. TTSV \citep{kang2025model} optimizes \textit{prefix} embeddings at test time via trajectory entropy minimization\rspadd{.} SoftCoT \citep{xu2025softcot} replaces explicit chain-of-thought tokens with soft tokens generated by an assistant model\rspadd{.} LTPO \citep{ye2025thinking} optimizes the \textit{soft prompt} per problem to maximize its own confidence on that problem\rspadd{.} COCONUT \citep{hao2024training} feeds its hidden states back as input embeddings\rspadd{.} MemGen \citep{zhang2025memgen} uses embeddings as experience memory\rspadd{.} Pause tokens \citep{goyal2024think} insert learnable tokens to provide additional computation steps.

These methods differ in position, training stage, and training target, yet \rspadd{assume that} \emph{learned representations} are the core driver of the effect. This paradigm faces several limitations. \textbf{(i)} Training is domain/benchmark-specific, requiring re-optimization for new settings and sensitive to seed/hyperparameter choices. \textbf{(ii)} Several have been observed to degrade on out-of-distribution challenging tasks~\citep{ye2025thinking}. \textbf{(iii)} Evidence for the effect is largely empirical, with limited theoretical explanation beyond the training objective. Furthermore, soft prompting and prefix-tuning are theoretically restricted to biasing attention outputs in a fixed direction, without changing the relative attention pattern among content tokens~\citep{petrov2024prefix}. Together, these findings cast doubt on whether the learned prompt vectors function as intended: they are suitable for \emph{eliciting} or \emph{combining} skills in the pretrained model, but struggle to learn new capabilities. \rspadd{These studies leave the contribution of embedding injection itself insufficiently distinguished from that of learned content. We investigate this distinction through RSP and controlled interventions on trained prompt vectors.}

\subsection{Noise Injection in Neural Networks}\label{sec:rw_perturbation}

Neural networks have incorporated noise at several stages and locations. During training, dropout \citep{srivastava2014dropout} randomly deactivates \textit{hidden activations}, and NEFTune \citep{jain2024neftune} \textit{adds} \rspadd{uniform} noise to \textit{input embeddings} during instruction fine-tuning to improve instruction following. \rspadd{During inference}, randomized smoothing \citep{cohen2019certified} injects Gaussian noise into \textit{raw inputs} to obtain certified robustness, and SmoothLLM \citep{robey2023smoothllm} perturbs prompts at the \textit{discrete character} level to defend against jailbreaking. In RL, NoisyNet \citep{fortunato2018noisy} \textit{adds} learnable noise to \textit{network weights} to aid exploration. Rather than \textit{adding} noise to embeddings, RSP \textit{appends} it, and differs in two respects: \textbf{(i)} \rspadd{unlike NoisyNet, it requires neither model fine-tuning nor learned noise parameters.} \textbf{(ii)} \rspadd{RSP does not require aggregating predictions across multiple perturbed inputs.}

%% file: sections/4_why_rsp_works.tex

\section{Random Soft Prompts and Why They Work}\label{sec:why_rsp_works}

\subsection{Random Soft Prompt}\label{sec:rsp_definition}

Random embedding vectors appended to LLM inputs, \rspadd{carrying no learned content}, achieve accuracy comparable to optimized \textit{soft prompt} methods across mathematical reasoning settings (Table~\ref{tab:prior_comparison}, \S\ref{sec:performance}). This section lays the theoretical groundwork for why a training-free distribution suffices, by tracking the \emph{attention mass} $w$ placed on RSP tokens. Early in decoding $w$ is large enough to open a different decoding branch (\S\ref{sec:improve}). As decoding proceeds, RSP attention empirically declines. Section~\ref{sec:converge} analyzes this behavior through the joint effects of KV cache growth and the evolving attention-logit gap. This constitutes an implicit explore-then-exploit process. Independent resampling provides a perturbation for each rollout, enabling exploration of different decoding trajectories (Appendix~\ref{app:proofs}).

\begin{table}[!t]
\centering
\caption{\rspadd{\textbf{Comparison with optimized \textit{soft prompt} methods.} Accuracy (\%) under unified evaluation, reported as mean $\pm$ std over five initialization seeds for RSP, TTSV, SoftCoT, and LTPO. The baseline uses greedy decoding. RSP uses \textit{suffix} injection, with baseline and RSP values matching Table~\ref{tab:performance}. \textbf{Bold} and \underline{underline} denote the highest and second-highest mean accuracy in each row, respectively.}}
\label{tab:prior_comparison}
\footnotesize
\setlength{\tabcolsep}{3pt}
\renewcommand{\arraystretch}{0.9}
\resizebox{\textwidth}{!}{%
\begin{tabular}{llccccc}
\toprule
\textbf{Model} & \textbf{Benchmark} & \textbf{Baseline} & \cellcolor{blue!10}\textbf{RSP (\textit{suffix})} & \textbf{TTSV} & \textbf{SoftCoT} & \textbf{LTPO} \\
\midrule
\multirow{3}{*}{Qwen2.5-Math-7B-Instruct} & MATH-500 & $83.20$ & \cellcolor{blue!10}$\mathbf{83.52}$\,{\tiny$\pm 0.56$} & $\underline{83.44}$\,{\tiny$\pm 0.55$} & $82.88$\,{\tiny$\pm 0.32$} & $82.64$\,{\tiny$\pm 0.43$} \\
 & GSM8K & $\underline{95.45}$ & \cellcolor{blue!10}$\mathbf{95.60}$\,{\tiny$\pm 0.08$} & $95.12$\,{\tiny$\pm 0.11$} & $94.74$\,{\tiny$\pm 0.28$} & $94.97$\,{\tiny$\pm 0.11$} \\
 & AIME24 & $13.33$ & \cellcolor{blue!10}$\underline{16.67}$\,{\tiny$\pm 0.00$} & $15.32$\,{\tiny$\pm 4.11$} & $\mathbf{18.00}$\,{\tiny$\pm 2.67$} & $12.67$\,{\tiny$\pm 1.33$} \\
\midrule
\multirow{3}{*}{Qwen2.5-Math-1.5B-Instruct} & MATH-500 & $73.00$ & \cellcolor{blue!10}$\mathbf{74.12}$\,{\tiny$\pm 1.36$} & $72.32$\,{\tiny$\pm 1.75$} & $72.88$\,{\tiny$\pm 2.27$} & $\underline{73.68}$\,{\tiny$\pm 0.97$} \\
 & GSM8K & $\mathbf{85.29}$ & \cellcolor{blue!10}$\underline{84.99}$\,{\tiny$\pm 0.22$} & $84.04$\,{\tiny$\pm 1.07$} & $82.97$\,{\tiny$\pm 2.97$} & $84.76$\,{\tiny$\pm 0.52$} \\
 & AIME24 & $\underline{10.00}$ & \cellcolor{blue!10}$\mathbf{13.33}$\,{\tiny$\pm 5.27$} & $6.02$\,{\tiny$\pm 1.22$} & $9.33$\,{\tiny$\pm 2.42$} & $9.33$\,{\tiny$\pm 1.33$} \\
\bottomrule
\end{tabular}}
\end{table}

To isolate the act of injection from learned content while keeping the magnitude comparable to real tokens and not privileging any direction in embedding space \rspadd{(Proposition~\ref{thm:minimax_directional_coverage})}, we draw RSP from an isotropic Gaussian fitted to the entrywise statistics of the pretrained embedding table. \rspadd{Alternative initializations score slightly lower on MATH-500 and AIME24, within seed noise (Appendix~\ref{app:law_scale}).} Let $\mathbf{W}_E \in \mathbb{R}^{V\times d}$ denote the pretrained token-embedding matrix, with $V$ the vocabulary size and $d$ the hidden dimension. Write its entrywise statistics as $\mu_E := \tfrac{1}{V\cdot d}\sum_{v,k}[\mathbf{W}_E]_{v,k}$ (the mean over all $V\cdot d$ entries) and $\sigma_E := \mathrm{std}(\mathbf{W}_E)$. An RSP of length $L$ is a sequence of continuous vectors $\mathbf{H} = [h_1,\dots,h_L]\in\mathbb{R}^{L\times d}$ drawn independently from
\begin{equation}\label{eq:rsp_def}
h_j \;\sim\; \mathcal{N}(\mu_E\mathbf{1}_d,\;\sigma_E^2\mathbf{I}_d), \quad j=1,\dots,L.
\end{equation}
The centered form $\bar h_j := h_j - \mu_E\mathbf{1}_d \sim \mathcal{N}(0,\sigma_E^2\mathbf{I}_d)$ is a zero-mean isotropic Gaussian. Letting $\mathbf{X} \in \mathbb{R}^{T\times d}$ denote the input token-embedding sequence (the $T$ rows of $\mathbf{W}_E$ corresponding to the prompt tokens), the main text uses the \textit{suffix} form $[\mathbf{X};\mathbf{H}] \in \mathbb{R}^{(T+L)\times d}$ as the default. Other positions (\textit{prefix} $[\mathbf{H};\mathbf{X}]$, \textit{infix} ($\mathbf{H}$ inserted within $\mathbf{X}$) \citep{kang2025model, xu2025softcot, ye2025thinking, zhang2025memgen}) are reported as ablations in \S\ref{sec:performance} and Appendix~\ref{app:ablation}. RSP is not a learnable parameter, so it receives no gradient and is freshly drawn for each rollout. We define decoding step $t$ to be a \emph{branching event} and the last-layer hidden state $\mathbf{s}^{(t)}$ to be a \emph{branching state} when the LM-head top-$K$ differs from that of the no-RSP baseline $\bar{\mathbf{s}}^{(t)}$ (Figure~\ref{fig:main}, red).\footnote{Italic $h_j$ denotes a vector in the RSP input space, and bold $\mathbf{s}^{(t)}$ denotes a hidden state. $K$ denotes the cardinality in top-$K$ analyses. Nucleus sampling uses a separate threshold $p$. $t$ for decoding step, $\ell$ for self-attention layer.}

\subsection{Exploration: RSP strengthens early-stage exploration}\label{sec:improve}\label{sec:multipath}

Inserting a random prompt concentrates attention on it, and that concentration amplifies exploration. The effect of a single injection on the hidden state raises two questions\rspadd{,} \emph{how much} and \emph{where}\rspadd{,} decided respectively by the scalar $w$ inside one attention head and by the distribution of $\mathbf{H}$. The isotropic Gaussian addresses both at once. Both questions trace back to the branching condition of \S\ref{sec:rsp_definition}: for a single RSP draw to shift top-$K$, its perturbation of the head output must be large enough to propagate through the layers (\emph{how much}) and along directions that affect logit ordering (\emph{where}).

Consider one attention head in self-attention layer $\ell$. At query position $i$, the query attends $n\ge1$ unmasked real tokens ($n$ depends on $i$ via the causal mask) \rspadd{and} $L\ge1$ RSP tokens, all attention logits finite. Write the attention weights as $\alpha_{ij}^{(\ell)}$ (softmax-normalized so that $\sum_{j=1}^n \alpha_{ij}^{(\ell)} + \sum_{j=1}^L \alpha_{i,n+j}^{(\ell)} = 1$) and the value vectors on the real and RSP sides as $\mathbf{v}_j^{(\ell)},\mathbf{v}_{r_j}^{(\ell)}$. Defining the attention mass on random tokens $w_{r,i}^{(\ell)} := \sum_{j=1}^L \alpha_{i,n+j}^{(\ell)}$, the real-token mass $1-w_{r,i}^{(\ell)}$ is positive and the head output $\mathbf{o}_i^{(\ell)}$ splits \emph{exactly} into a renormalized real-token term $\tilde{\mathbf{o}}_i^{(\ell)}$ and an RSP-induced contribution $\boldsymbol{\eta}_i^{(\ell)}$:
\begin{equation}\label{eq:decomp}
\mathbf{o}_i^{(\ell)}
\;=\;
(1-w_{r,i}^{(\ell)})\,\tilde{\mathbf{o}}_i^{(\ell)}
\;+\;
\boldsymbol{\eta}_i^{(\ell)},
\quad
\tilde{\mathbf{o}}_i^{(\ell)} := \sum_{j=1}^n \tfrac{\alpha_{ij}^{(\ell)}}{1-w_{r,i}^{(\ell)}}\,\mathbf{v}_j^{(\ell)},
\quad
\boldsymbol{\eta}_i^{(\ell)} := \sum_{j=1}^L \alpha_{i,n+j}^{(\ell)}\mathbf{v}_{r_j}^{(\ell)}.
\end{equation}

\rspadd{Equation~\eqref{eq:decomp} follows exactly by factoring the positive real-token attention mass $1-w_{r,i}^{(\ell)}$ out of the real-token contribution.} This single quantity $w_{r,i}^{(\ell)}$ controls both the attenuation ratio $1-w_{r,i}^{(\ell)}$ on the real signal and the upper bound on the random contribution $\|\boldsymbol{\eta}_i^{(\ell)}\| \le w_{r,i}^{(\ell)}\,\max_j \|\mathbf{v}_{r_j}^{(\ell)}\|$.

\rspadd{The scalar $w$ controls \emph{how much}, while the distribution of $\mathbf{H}$ determines \emph{where}.} For the local logit argument, isolate one RSP position and let $\bar h \in \mathbb{R}^d$ denote the i.i.d.\ marginal of any centered vector $\bar h_j$ in Eq.~\eqref{eq:rsp_def}, with the other centered RSP positions fixed at zero. Let $z_a(\bar h)$ be the under-RSP output logit at vocab token $a$ and write the vocab-logit gap $\Delta_{ab}(\bar h) := z_a(\bar h) - z_b(\bar h)$ ($\bar h = 0$ is the centered mean, not the no-RSP state\rspadd{, as explained in} Appendix~\ref{app:why_random}). The transformer's logit map is non-linear in $\bar h$, but a first-order Taylor expansion around $\bar h = 0$ gives the local surrogate $\Delta_{ab}(\bar h) \approx \Delta_{ab}(0) + b_{ab}^\top \bar h$, with $b_{ab} := \nabla_{\bar h}(z_a - z_b)|_{\bar h=0} \in \mathbb{R}^d$ the gradient direction along which vocab tokens $a, b$ swap rank (not unit-norm in general). Because RSP is training-free, we cannot know in advance which $b_{ab}$ opens a useful branch, so the distribution must avoid systematically under-perturbing any direction. A deterministic injection pins to one direction\rspadd{, while vocabulary} sampling inherits the embedding table's anisotropy. \rspadd{At a fixed variance budget, the worst-case directional variance is maximized if and only if the covariance is isotropic.}

\begin{proposition}[Maximin directional coverage]\label{thm:minimax_directional_coverage}
Let $\mathcal{D}_\rho$ denote the family of zero-mean distributions on $\mathbb{R}^d$ whose covariance satisfies $\mathrm{tr}(\Sigma_D)\le\rho^2$. For each $D \in \mathcal{D}_\rho$, $\min_{\|u\|=1}\mathrm{Var}_{h\sim D}(u^\top h) = \lambda_{\min}(\Sigma_D)$, and $\sup_{D \in \mathcal{D}_\rho} \lambda_{\min}(\Sigma_D) = \rho^2/d$, attained iff $\Sigma_D = (\rho^2/d)\mathbf{I}_d$.
\end{proposition}

This is a \emph{design criterion} for the prompt law, not a guarantee of correctness\rspadd{. Correctness} enters through the task-side $p_{\min}(x)$ assumption of \S\ref{sec:converge} and independent resampling.\footnote{The maximin is stated in input space\rspadd{. Isotropy} is not claimed to survive transformer layers. Any logit direction pulls back via the model Jacobian to an input-space $b$, and isotropy guarantees only that no such pulled-back direction (in the role of $u$ above) has zero projected variance.}

Gaussian RSP attains the equality with $\Sigma = \sigma_E^2 \mathbf{I}_d$ ($\rho^2 = d\sigma_E^2$) and adds \emph{full support} on $\mathbb{R}^d$ (Appendix~\ref{app:why_random}). Isotropy excludes no direction\rspadd{. Full} support gives positive probability to the open branching-event region (\S\ref{sec:rsp_definition}) whenever it is non-empty (Appendix~\ref{app:topk_entry}). \rspadd{Monotone temperature preserves ranking and cannot reach this region. Together, these properties allow RSP to change token rankings and reach different branching states across independent draws early in reasoning.}

\subsection{Annealing: KV cache growth dilutes RSP attention}\label{sec:converge}

\rspadd{The early branching stabilizes as decoding proceeds. Each step adds one real token while the number of RSP tokens stays fixed. To characterize the resulting attention allocation, we express the RSP attention mass exactly in terms of token counts and a generation-dependent attention-logit gap.}

{\ifrspchanges\color{rspadded}\fi
\begin{theorem}[RSP attention mass]\label{thm:decay}
Under the attention setup of \S\ref{sec:improve}, suppress the layer index $\ell$ and let $d_{\text{head}}$ denote the per-head key dimension. Write the scaled logits as $s_{ij}:=\mathbf{q}_i^\top\mathbf{k}_j/\sqrt{d_{\text{head}}}$ and $s_{ir_j}:=\mathbf{q}_i^\top\mathbf{k}_{r_j}/\sqrt{d_{\text{head}}}$. The token-count-normalized log-sum-exp gap and the RSP attention mass are
\begin{equation}
g_i := \log\!\Bigl(\tfrac{1}{n}\textstyle\sum_{j=1}^{n}e^{s_{ij}}\Bigr)
       - \log\!\Bigl(\tfrac{1}{L}\textstyle\sum_{j=1}^{L}e^{s_{ir_j}}\Bigr),
\qquad
w_{r,i} = \frac{L}{L+n e^{g_i}}.
\label{eq:rsp_gap_mass}
\end{equation}

\end{theorem}
\begin{proof}
Let $S_i:=\sum_{j=1}^{n}e^{s_{ij}}$ and $R_i:=\sum_{j=1}^{L}e^{s_{ir_j}}$. By definition, $e^{g_i}=(S_i/n)/(R_i/L)$, so softmax normalization gives $w_{r,i}=R_i/(S_i+R_i)=L/(L+n e^{g_i})$.
\end{proof}
}

\begin{figure}[!t]
  \centering
  \includegraphics[width=\textwidth]{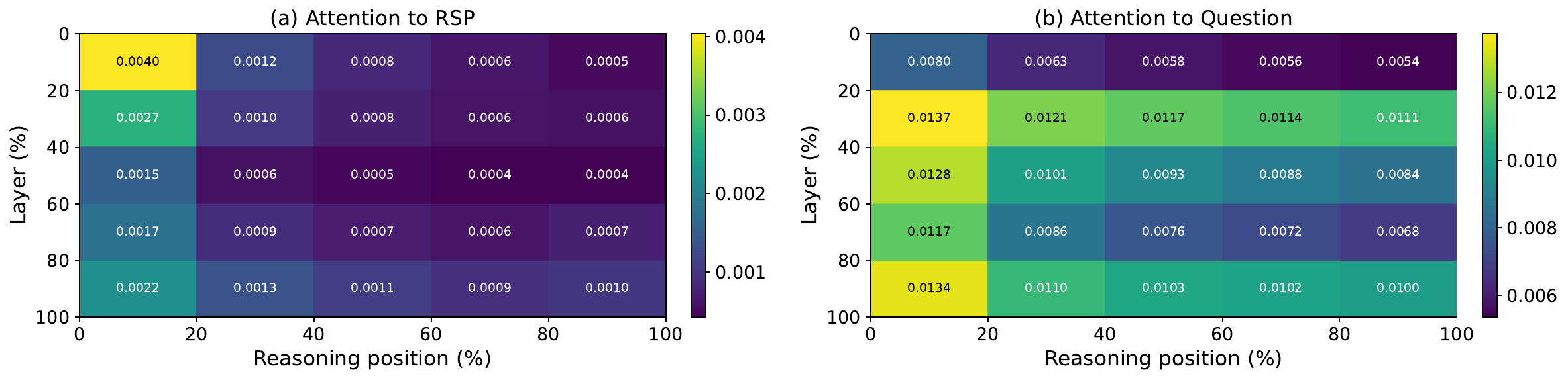}
  \caption{\rspadd{\textbf{Per-token attention mass during generation.} Results on Qwen2.5-Math-7B with \textit{suffix} injection, averaged over 500 MATH-500 problems with independent RSP draws.} (a) RSP-token attention decreases along the reasoning-position axis, \rspadd{consistent with the token-count and gap dynamics described by Theorem~\ref{thm:decay}.} (b) Question-token attention stays comparatively uniform. \rspadd{Both panels report attention mass per token, averaged over heads and samples, normalized by the RSP token count $L$ in (a) and the question-token count in (b) (Appendix~\ref{app:attention_mass}).}}
  \label{fig:attention_mass}
\end{figure}

\rspadd{Applied at each layer $\ell$, Eq.~\eqref{eq:rsp_gap_mass} makes the generation dependence explicit: with $L$ fixed, RSP attention decreases as $n e^{g_i^{(\ell)}}$ increases. Cache growth increases $n$, while the evolving queries and keys determine $g_i^{(\ell)}$. Under the same setting as Figure~\ref{fig:attention_mass}, the measured gap broadly increases along generation, most clearly in the shallowest and deepest layer bands (Appendix~\ref{app:gap_dynamics}). Together with the observed attention decay, this supports the annealing picture. Through Eq.~\eqref{eq:decomp}, declining RSP attention reduces the total weight assigned to RSP value vectors, consistent with early exploration followed by stabilization.}

Lifting single-shot branching to task accuracy across $N$ rollouts requires (M1) a positive lower bound $p_{\min}(x)>0$ on the per-rollout probability that execution reaches a correct trajectory and (M2) the $N$ executions are independent. \rspadd{Then,} $\Pr(\mathrm{Pass@}N \text{ on } x) \ge 1 - (1-p_{\min}(x))^N$ (Appendix~\ref{app:passN_ensemble}). \S\ref{sec:improve}'s isotropy and full support make local branch opening reachable when the open region is nonempty, but they do not by themselves guarantee the task-side correctness condition (M1)\rspadd{.} \emph{\rspadd{Independent} resampling} supplies (M2). \rspadd{Section~\ref{sec:experimental_analysis} evaluates the contribution of independent RSP resampling to Pass@$N$, comparing shared and independently resampled RSPs under temperature sampling.}

%% file: sections/3_experimental_analysis.tex

\section{Empirical Validation}\label{sec:experimental_analysis}

We first evaluate whether untrained RSPs reach the performance range of optimized soft prompts (\S\ref{sec:performance}), then directly test the role of learned content through controlled interventions on the injected vectors (\S\ref{sec:content_probes}). We next examine how injection reshapes the output distribution (\S\ref{sec:entropy}) and whether these changes translate into trajectory diversity and Pass@$N$ gains under independent resampling (\S\ref{sec:diversity}). \rspadd{Finally, \S\ref{sec:application} evaluates whether RSP improves correct-solution collection and subsequent SFT.}

\subsection{Experimental Setup}\label{sec:setup}

We evaluate on three mathematical reasoning benchmarks, MATH-500 \citep{lightman2024lets}, GSM8K \citep{cobbe2021training}, and AIME24, using LLaMA-3.1-8B-Instruct \citep{grattafiori2024llama} and the Qwen2.5-Math model family \citep{yang2024qwen25math} across instruct models and base models with mismatched chat templates. \rspadd{We additionally evaluate the general-purpose Qwen2.5-\{1.5B, 7B\}-Instruct models on code generation (HumanEval+, MBPP+) \citep{liu2023evalplus}, moral reasoning (ETHICS-commonsense, ETHICS-justice) \citep{hendrycks2021ethics}, general language understanding (MMLU) \citep{hendrycks2021mmlu}, and commonsense reasoning (PIQA) \citep{bisk2020piqa}.} Tables~\ref{tab:performance} and~\ref{tab:prior_comparison} use greedy decoding to isolate the effect of RSPs from sampling stochasticity\rspadd{. The} diversity and Pass@$N$ experiments in \S\ref{sec:diversity} use the sampling settings specified there. RSPs follow the definition of \S\ref{sec:rsp_definition}, using the default \textit{suffix} position with a drawn $\mathbf{H}$ per rollout. Since Theorem~\ref{thm:decay} \rspadd{expresses RSP attention mass in terms of the RSP length $L$, real-token count, and attention-logit gap}, \rspadd{we compare $L \in \{10,15,20\}$ for each model and benchmark to identify an appropriate perturbation strength.} \rspadd{Tables~\ref{tab:performance} and~\ref{tab:prior_comparison} report results using the selected lengths.} Table~\ref{tab:performance} reports mean accuracy and standard deviation over five RSP initialization seeds\rspadd{.} \rspadd{For TTSV, SoftCoT, and LTPO, we repeat training or optimization with five initialization seeds while keeping the hyperparameters in Appendix~\ref{app:prior_hp} fixed.} The mechanism analyses in \S\ref{sec:entropy}--\S\ref{sec:diversity} use a fixed setting ($L{=}20$, \textit{suffix}), which insulates them from this selection effect. Full experimental details and per-position results are in Appendix~\ref{app:setup}.

\begin{table}[!t]
\centering
\caption{\rspadd{\textbf{Accuracy across model configurations and benchmarks.} RSP accuracy (\%) is reported as mean $\pm$ std over five initialization seeds. The baseline uses greedy decoding.}}
\label{tab:performance}
\footnotesize
\setlength{\tabcolsep}{3pt}
\renewcommand{\arraystretch}{0.9}
\sbox{\rspPerformanceTable}{%
\begin{tabular}{l ccc ccc ccc}
\toprule
 & \multicolumn{3}{c}{\textbf{MATH-500}} & \multicolumn{3}{c}{\textbf{GSM8K}} & \multicolumn{3}{c}{\textbf{AIME24}} \\
\cmidrule(lr){2-4} \cmidrule(lr){5-7} \cmidrule(lr){8-10}
\textbf{Model} & Baseline & \cellcolor{blue!10}RSP(\textit{prefix}) & \cellcolor{blue!10}RSP(\textit{suffix}) & Baseline & \cellcolor{blue!10}RSP(\textit{prefix}) & \cellcolor{blue!10}RSP(\textit{suffix}) & Baseline & \cellcolor{blue!10}RSP(\textit{prefix}) & \cellcolor{blue!10}RSP(\textit{suffix}) \\
\midrule
\multicolumn{10}{l}{\textit{Instruct Models}} \\
\midrule
LLaMA-3.1-8B-Inst & 52.20 & \cellcolor{blue!10}$41.24$\,{\tiny$\pm 2.16$} & \cellcolor{blue!10}$49.52$\,{\tiny$\pm 1.48$} & 85.75 & \cellcolor{blue!10}$76.30$\,{\tiny$\pm 0.75$} & \cellcolor{blue!10}$86.70$\,{\tiny$\pm 0.39$} & 6.67 & \cellcolor{blue!10}$3.33$\,{\tiny$\pm 2.36$} & \cellcolor{blue!10}$6.67$\,{\tiny$\pm 3.33$} \\
Qwen2.5-Math-7B-Inst & 83.20 & \cellcolor{blue!10}$81.72$\,{\tiny$\pm 0.46$} & \cellcolor{blue!10}$83.52$\,{\tiny$\pm 0.56$} & 95.45 & \cellcolor{blue!10}$94.92$\,{\tiny$\pm 0.28$} & \cellcolor{blue!10}$95.60$\,{\tiny$\pm 0.08$} & 13.33 & \cellcolor{blue!10}$13.33$\,{\tiny$\pm 2.36$} & \cellcolor{blue!10}$16.67$\,{\tiny$\pm 0.00$} \\
Qwen2.5-Math-1.5B-Inst & 73.00 & \cellcolor{blue!10}$75.60$\,{\tiny$\pm 0.51$} & \cellcolor{blue!10}$74.12$\,{\tiny$\pm 1.36$} & 85.29 & \cellcolor{blue!10}$85.26$\,{\tiny$\pm 0.39$} & \cellcolor{blue!10}$84.99$\,{\tiny$\pm 0.22$} & 10.00 & \cellcolor{blue!10}$12.00$\,{\tiny$\pm 1.83$} & \cellcolor{blue!10}$13.33$\,{\tiny$\pm 5.27$} \\
\midrule
\multicolumn{10}{l}{\textit{Base Models + ChatML (Format Mismatch)}} \\
\midrule
Qwen2.5-Math-7B & 52.20 & \cellcolor{blue!10}$57.40$\,{\tiny$\pm 1.71$} & \cellcolor{blue!10}$69.40$\,{\tiny$\pm 1.03$} & 58.30 & \cellcolor{blue!10}$56.25$\,{\tiny$\pm 1.14$} & \cellcolor{blue!10}$75.88$\,{\tiny$\pm 0.36$} & 23.33 & \cellcolor{blue!10}$9.33$\,{\tiny$\pm 4.35$} & \cellcolor{blue!10}$20.00$\,{\tiny$\pm 2.36$} \\
Qwen2.5-Math-1.5B & 34.40 & \cellcolor{blue!10}$61.56$\,{\tiny$\pm 1.42$} & \cellcolor{blue!10}$48.80$\,{\tiny$\pm 2.09$} & 37.45 & \cellcolor{blue!10}$66.07$\,{\tiny$\pm 1.03$} & \cellcolor{blue!10}$50.24$\,{\tiny$\pm 2.09$} & 13.33 & \cellcolor{blue!10}$14.67$\,{\tiny$\pm 2.98$} & \cellcolor{blue!10}$8.00$\,{\tiny$\pm 5.06$} \\
\bottomrule
\end{tabular}}
\global\rspPerformanceTableWidth=\wd\rspPerformanceTable
\resizebox{\textwidth}{!}{\usebox{\rspPerformanceTable}}
\end{table}

\subsection{How Does Untrained RSP Compare to Baselines \& Optimized Soft Prompts?}\label{sec:performance}

\paragraph{\rspadd{Mathematical reasoning.}}
Compared with three optimized methods (TTSV, SoftCoT, LTPO) under a unified evaluation pipeline (Appendix~\ref{app:extraction}), RSP reaches a similar accuracy range without learning the injected vectors (Table~\ref{tab:prior_comparison}, presented in \S\ref{sec:why_rsp_works}). \rspadd{What we want to highlight is the converse: a fixed, training-free distribution reaching this range provides evidence that part of the \textit{soft-prompt} effect comes less from learned content than from the random directional perturbation that injection induces.}

\rspadd{Beyond the direct comparison with optimized methods, we further evaluate RSP across two injection positions (\textit{prefix}, \textit{suffix}) and several model configurations, unpacking the per-configuration patterns of Table~\ref{tab:performance} with five RSP initialization seeds. Where the baseline is already strong, \textit{suffix} RSP achieves comparable accuracy on the Qwen2.5-Math-Instruct models. On the harder AIME24 benchmark, \textit{suffix} RSP improves mean accuracy by $3.3$\,pp at both model scales. Another striking pattern is format-mismatch recovery: when a ChatML template is applied to a base model not trained for that format, \textit{suffix} recovers $+17.20/+17.58$\,pp on Qwen2.5-Math-7B (MATH-500/GSM8K), and \textit{prefix} recovers $+27.16/+28.62$\,pp on Qwen2.5-Math-1.5B.}

\paragraph{\rspadd{Evaluation on additional benchmarks.}}
\rspadd{Prior soft-prompt work primarily reports results on mathematical reasoning. We extend the evaluation to six additional benchmarks across four domains: code, moral reasoning, general language understanding, and commonsense. We use the general-purpose Qwen2.5-\{1.5B, 7B\}-Instruct models and average results over five seeds.}

\input{sections/3_cross_domain_table}

\rspadd{Table~\ref{tab:cross_domain} shows consistent gains on code ($+2.1$ to $+6.2$\,pp on HumanEval+/MBPP+), with the 1.5B model benefiting most. On short-answer MCQ tasks (ETHICS, MMLU, and PIQA), RSP performs comparably to the baseline. Together with the AIME24 results, these gains are consistent with the mechanism described in Section~\ref{sec:why_rsp_works}: RSP can perturb early decoding decisions and, across independent draws, enable exploration of different reasoning trajectories. This may benefit tasks involving extended reasoning and generation, such as hard mathematics and code. Section~\ref{sec:diversity} directly examines whether this trajectory diversity translates into Pass@$N$ gains.}

\input{sections/3_content_probes}

\subsection{How Does RSP Affect the Output Distribution?}\label{sec:entropy}

\rspadd{The results in Section~\ref{sec:content_probes} motivate examining whether RSP and optimized soft prompts share common patterns in their output distributions despite differences in learned content and training objectives. We compare their generation dynamics using Qwen2.5-Math-1.5B-Instruct on MATH-500, measuring per-token entropy, top-1 probability, and varentropy across generation steps. Figure~\ref{fig:entropy_first5} compares these metrics over the first 5\% of generation steps for RSP and prior \textit{soft prompt} methods (LTPO, SoftCoT, TTSV). Full curves are provided in Appendix~\ref{app:entropy_full}.}

\begin{figure}[ht]
  \centering
  \includegraphics[width=\textwidth]{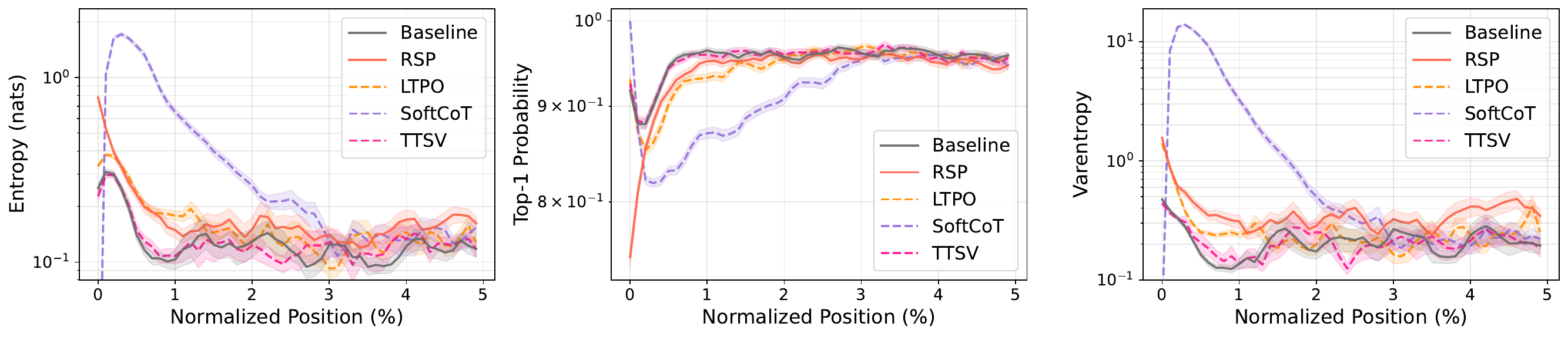}
  \caption{Mean entropy, top-1 probability, and varentropy during the first 5\% of generation steps (Qwen2.5-Math-1.5B-Instruct, MATH-500). Shading indicates $\pm$1 SEM. Solid lines: \rspadd{RSP} and the baseline. Dashed lines: prior \textit{soft prompt} methods (LTPO, SoftCoT, TTSV).}
  \label{fig:entropy_first5}
\end{figure}

\rspadd{RSP and optimized soft prompt methods exhibit similar temporal patterns despite their different training objectives. RSP shows early-stage entropy${\uparrow}$, top-1 probability${\downarrow}$, and varentropy${\uparrow}$. When continuous embeddings outside the vocabulary enter the input, attention to the additional keys can redistribute probability mass across multiple high-probability candidates in the early next-token distribution. The joint pattern is consistent with this redistribution, a precondition for the trajectory branching of \S\ref{sec:improve}. The optimized methods show related early behavior. Even TTSV, which explicitly minimizes mean trajectory entropy, exhibits early entropy comparable to the baseline. Later in generation, the metrics converge toward the baseline, consistent with the attenuation analyzed in \S\ref{sec:converge}.}

\subsection{Does RSP Induce Trajectory Diversity?}\label{sec:diversity}

\rspadd{Section~\ref{sec:entropy} showed that RSP reshapes the early-stage output distribution. We now examine whether this shift is accompanied by greater diversity in token rankings, the semantic content of reasoning trajectories, and task-level outcomes. These analyses connect changes in local token distributions with diversity across complete trajectories and gains in Pass@$N$. We use Qwen2.5-Math-1.5B-Instruct on MATH-500 with \textit{suffix} injection and $L=20$, and additionally evaluate Pass@$N$ on AIME24.}

\paragraph{Token-level: distribution beyond temperature.}
\begin{wraptable}{R}{0.55\textwidth}
  \centering
  \caption{First-token distribution metrics measured against the baseline at $\tau{=}1.0$, comparing baselines at higher temperatures ($\tau{=}2.0, 3.0$) with RSP at $\tau{=}1.0$. RSP metrics are averaged over 16 random seeds, and Acc reports mean $\pm$ std across 16 samples per problem.}
  \label{tab:rsp_vs_temp}
  \footnotesize
  \setlength{\tabcolsep}{3.5pt}
  \renewcommand{\arraystretch}{0.9}
  \begin{tabular}{lccc}
    \toprule
    \textbf{Metric} & \textbf{$\tau{=}2.0$} & \textbf{$\tau{=}3.0$} & \cellcolor{blue!10}\textbf{RSP} \\
    \midrule
    Spearman $\rho$ & 1.000 & 1.000 & \cellcolor{blue!10}\textbf{0.709} \\
    Mass outside top-10 & 5.18\% & 29.11\% & \cellcolor{blue!10}\textbf{21.86\%} \\
    JS divergence & 0.060 & 0.218 & \cellcolor{blue!10}\textbf{0.131} \\
    Acc (\%) & 20.77 $\pm$ 1.47 & 1.42 $\pm$ 0.35 & \cellcolor{blue!10}\textbf{73.30 $\pm$ 1.07} \\
    \bottomrule
  \end{tabular}
\end{wraptable}
We quantify how much RSP shifts the first-token distribution\rspadd{,} the critical forking point for reasoning trajectories \citep{wang2025beyond}. Using the baseline at $\tau{=}1.0$ as reference, we compare baselines at $\tau{=}2.0,3.0$ (16 samples per problem) against RSP at $\tau{=}1.0$ (averaged over 16 seeds) on three metrics: Spearman rank correlation $\rho$ to detect rank reorganization, the probability mass placed outside the baseline's top-10 to measure support expansion, and Jensen--Shannon divergence to quantify the overall magnitude of distributional change.

Table~\ref{tab:rsp_vs_temp} shows the contrast: RSP's distributional shift falls between $\tau{=}2.0$ and $\tau{=}3.0$ in magnitude, but the mechanism differs\rspadd{. Temperature} is a monotone transform that preserves ranking ($\rho=1$ at any $\tau$), while RSP partially preserves but reranks ($\rho=0.709$). \rspadd{The accuracy row} exposes the cost: matching RSP's magnitude with temperature toward $\tau{=}3.0$ collapses accuracy to $1.42\%$, whereas RSP preserves $73.30\%$. Thus, RSP produces a fundamentally different perturbation than temperature scaling. Moreover, RSP's accuracy variability across 16 seeds at $\tau{=}1.0$ is $\pm 1.07$, comparable to the $\pm 1.47$ that temperature sampling produces across 16 samples at $\tau{=}2.0$ (\rspadd{details} in Appendix~\ref{app:token_metrics}).

\paragraph{Trajectory-level: semantic diversity.}
\begin{wraptable}{r}{0.34\textwidth}
  \vspace{-1.2em}
  \centering
  \caption{Semantic diversity metrics (64 problems, 300 trajectories per condition).}
  \label{tab:semantic_metrics}
  \footnotesize
  \setlength{\tabcolsep}{4pt}
  \renewcommand{\arraystretch}{0.9}
  \begin{tabular}{lcc}
    \toprule
    \textbf{Metric} & \textbf{Baseline} & \cellcolor{blue!10}\textbf{RSP} \\
    \midrule
    Pairwise cos.\ dist. & 0.0612 & \cellcolor{blue!10}\textbf{0.0879} \\
    Inter-cluster dist. & 0.0183 & \cellcolor{blue!10}\textbf{0.0982} \\
    Intra-cluster dist. & 0.0526 & \cellcolor{blue!10}0.0528 \\
    \bottomrule
  \end{tabular}
\end{wraptable}
Token-level reranking raises a second question: do first-token perturbations diverge into distinct trajectories, or converge to similar ones? We sample 64 problems from MATH-500 and generate 300 independent trajectories per problem under Baseline and RSP at $\tau = 1.0$. Each trajectory is encoded by BGE-M3 \citep{chen2024bge} into a dense semantic vector, and we compute three metrics: \textit{pairwise cosine distance} (overall semantic spread), \textit{inter-cluster distance} between centroids of DBSCAN \citep{ester1996density}, a density-based clustering algorithm (separation between distinct trajectory groups), and \textit{intra-cluster distance} (coherence within groups). Table~\ref{tab:semantic_metrics} reveals three concurrent patterns: pairwise distance rises by $1.4\times$, inter-cluster distance jumps by $5.4\times$, and intra-cluster distance is essentially unchanged\rspadd{. The} trajectory set becomes more diverse and splits into more separated groups while preserving within-group coherence. RSP also yields larger pairwise distance than baseline on 56 of 64 problems.

\paragraph{Outcome-level: Pass@$N$ scaling.}
Does this token- and trajectory-level diversity translate into task-level outcome diversity, measured by Pass@$N$ \citep{chen2021codex} (the probability that at least one of $N$ sampled solutions is correct)? \rspadd{We generate 16 samples per problem on MATH-500 and AIME24 and compare four conditions. (i) \textit{Baseline} uses $\tau=1.0$. (ii) \textit{RSP (single seed)} uses $\tau=1.0$ with one RSP shared across all samples. (iii) \textit{RSP (indep.\ seed)} uses greedy decoding with a fresh RSP per sample. (iv) \textit{RSP (indep.\ seed)} uses $\tau=1.0$ with a fresh RSP per sample.}

\begin{figure}[ht]
  \centering
  \includegraphics[width=\textwidth]{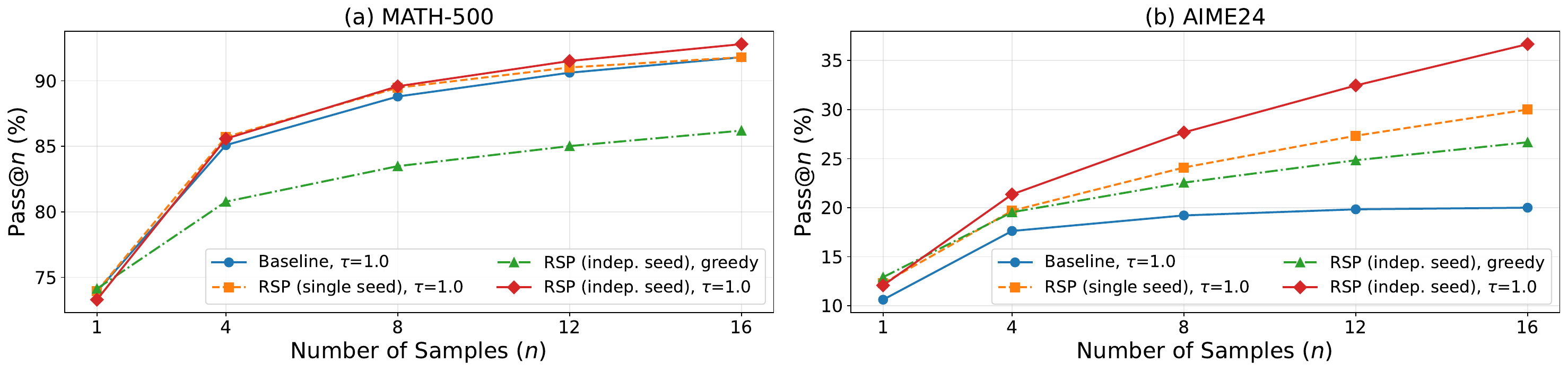}
  \caption{Pass@$N$ scaling on (a) MATH-500 and (b) AIME24 with Qwen2.5-Math-1.5B-Instruct, 16 samples per problem. \textit{Baseline}: temperature sampling only. \textit{RSP (single seed)}: single RSP shared across samples combined with temperature. \textit{RSP (indep.\ seed)}: a different RSP per sample, with or without temperature.}
  \label{fig:pass_at_n}
\end{figure}

\rspadd{Condition \textit{(iv)} achieves the highest Pass@$N$ on both benchmarks, reaching $92.80\%$ on MATH-500 and $36.67\%$ on AIME24 at $N=16$ \rspadd{(intermediate and annealed temperatures do not close the gap, Appendix~\ref{app:temperature})}. With a shared RSP, condition \textit{(ii)} matches the baseline on MATH-500 and improves AIME24 from $20.00\%$ to $30.00\%$, but remains below independent resampling on both benchmarks.} Pass@1 stays comparable to the baseline on MATH-500, and on the harder AIME24 \textit{(iv)} even surpasses it, suggesting that independent RSPs combined with temperature sampling expand trajectory diversity without sacrificing single-sample accuracy. Qualitative analysis is in Appendix~\ref{app:qualitative}.

%% file: sections/3_cross_domain_table.tex
\begin{table}[!ht]
\centering
\caption{\rspadd{\textbf{Accuracy beyond mathematics.} Results on Qwen2.5-\{1.5B, 7B\}-Instruct. RSP accuracy (\%) is reported as mean $\pm$ std over five initialization seeds. ETHICS-C and ETHICS-J denote ETHICS-commonsense and ETHICS-justice.}}
\label{tab:cross_domain}
\begingroup
\ifrspchanges\color{rspadded}\fi
\footnotesize
\setlength{\tabcolsep}{3pt}
\renewcommand{\arraystretch}{0.9}
\resizebox{\textwidth}{!}{%
\begin{tabular*}{\rspPerformanceTableWidth}{@{\extracolsep{\fill}}llcccccc@{}}
\toprule
\textbf{Model} & \textbf{Method} & \textbf{HumanEval+} & \textbf{MBPP+} & \textbf{ETHICS-C} & \textbf{ETHICS-J} & \textbf{MMLU} & \textbf{PIQA} \\
\midrule
Qwen2.5-1.5B & Baseline & 45.70 & 48.10 & 60.82 & 60.36 & 61.34 & 77.15 \\
 & RSP(\textit{prefix}) & $51.92\,{\scriptscriptstyle\pm\,1.61}$ & $51.60\,{\scriptscriptstyle\pm\,0.54}$ & $63.74\,{\scriptscriptstyle\pm\,0.37}$ & $59.68\,{\scriptscriptstyle\pm\,0.23}$ & $61.13\,{\scriptscriptstyle\pm\,0.13}$ & $77.17\,{\scriptscriptstyle\pm\,0.25}$ \\
 & RSP(\textit{suffix}) & $51.60\,{\scriptscriptstyle\pm\,1.98}$ & $52.42\,{\scriptscriptstyle\pm\,0.88}$ & $59.63\,{\scriptscriptstyle\pm\,0.27}$ & $60.60\,{\scriptscriptstyle\pm\,0.11}$ & $61.22\,{\scriptscriptstyle\pm\,0.23}$ & $76.25\,{\scriptscriptstyle\pm\,0.20}$ \\
\midrule
Qwen2.5-7B & Baseline & 71.30 & 65.60 & 76.65 & 76.59 & 73.14 & 87.43 \\
 & RSP(\textit{prefix}) & $73.80\,{\scriptscriptstyle\pm\,1.20}$ & $67.70\,{\scriptscriptstyle\pm\,0.27}$ & $76.81\,{\scriptscriptstyle\pm\,0.27}$ & $74.96\,{\scriptscriptstyle\pm\,0.39}$ & $72.90\,{\scriptscriptstyle\pm\,0.12}$ & $86.59\,{\scriptscriptstyle\pm\,0.22}$ \\
 & RSP(\textit{suffix}) & $74.16\,{\scriptscriptstyle\pm\,1.00}$ & $69.86\,{\scriptscriptstyle\pm\,0.47}$ & $76.64\,{\scriptscriptstyle\pm\,0.08}$ & $76.91\,{\scriptscriptstyle\pm\,0.12}$ & $73.29\,{\scriptscriptstyle\pm\,0.09}$ & $87.59\,{\scriptscriptstyle\pm\,0.10}$ \\
\bottomrule
\end{tabular*}}
\endgroup
\end{table}

%% file: sections/3_content_probes.tex
\subsection{Does Performance Depend on Learned Prompt Content?}\label{sec:content_probes}

\rspadd{Building on the results in \S\ref{sec:performance}, we examine the contributions of learned prompt content and embedding injection in greater detail. Following the settings in the original papers, we keep the injection token count and position fixed within each method while varying the injected vectors:} 20 \textit{prefix} tokens for TTSV, 4 \textit{infix} tokens for SoftCoT, and 8 \textit{infix} tokens for LTPO. We evaluate three complementary probes \rspadd{(i)--(iii)} on Qwen2.5-Math-1.5B-Instruct across MATH-500, GSM8K, and AIME24 (Tables~\ref{tab:content_probes} and~\ref{tab:overtrained_probes}). All results report mean accuracy and standard deviation over three seeds \rspadd{(seed details in Appendix~\ref{app:prior_hp}).} \rspadd{Appendix~\ref{app:content_probes_7b} reports additional results on Qwen2.5-Math-7B-Instruct.}

\begin{table}[!t]
  \centering
  \caption{\rspadd{\textbf{Intervening on learned prompt content.} Results on Qwen2.5-Math-1.5B-Instruct. Accuracy (\%) is reported as mean $\pm$ std over three seeds. Injection length and position remain fixed within each method. \textbf{Bold} denotes the highest mean accuracy in each column.}}
  \label{tab:content_probes}
  \footnotesize
  \setlength{\tabcolsep}{1.5pt}
  \renewcommand{\arraystretch}{0.9}
  \resizebox{\textwidth}{!}{%
  \begin{tabular*}{\rspPerformanceTableWidth}{@{\extracolsep{\fill}}lccc@{\hspace{5pt}}ccc@{\hspace{5pt}}ccc@{}}
    \toprule
    & \multicolumn{3}{c}{\textbf{TTSV}} & \multicolumn{3}{c}{\textbf{LTPO}} & \multicolumn{3}{c}{\textbf{SoftCoT}} \\
    \cmidrule(lr){2-4}\cmidrule(lr){5-7}\cmidrule(l){8-10}
    \textbf{Intervention} & MATH-500 & GSM8K & AIME24 & MATH-500 & GSM8K & AIME24 & MATH-500 & GSM8K & AIME24 \\
    \midrule
    \textbf{Trained} & $72.33$\,{\tiny\color{black!65}$\pm 1.46$} & $84.08$\,{\tiny\color{black!65}$\pm 1.08$} & $6.67$\,{\tiny\color{black!65}$\pm 0.00$} & $74.07$\,{\tiny\color{black!65}$\pm 0.82$} & $84.69$\,{\tiny\color{black!65}$\pm 0.21$} & $7.78$\,{\tiny\color{black!65}$\pm 1.57$} & $72.87$\,{\tiny\color{black!65}$\pm 1.18$} & $84.84$\,{\tiny\color{black!65}$\pm 1.24$} & $8.89$\,{\tiny\color{black!65}$\pm 1.57$} \\
    \midrule
    50\% front & $\mathbf{75.40}$\,{\tiny\color{black!65}$\pm 0.28$} & $84.99$\,{\tiny\color{black!65}$\pm 0.70$} & $10.00$\,{\tiny\color{black!65}$\pm 0.00$} & $\mathbf{74.27}$\,{\tiny\color{black!65}$\pm 0.47$} & $84.63$\,{\tiny\color{black!65}$\pm 0.34$} & $10.00$\,{\tiny\color{black!65}$\pm 2.72$} & $72.47$\,{\tiny\color{black!65}$\pm 0.98$} & $84.15$\,{\tiny\color{black!65}$\pm 1.29$} & $10.00$\,{\tiny\color{black!65}$\pm 2.72$} \\
    50\% random & $72.93$\,{\tiny\color{black!65}$\pm 1.95$} & $81.07$\,{\tiny\color{black!65}$\pm 3.54$} & $\mathbf{12.22}$\,{\tiny\color{black!65}$\pm 1.57$} & $73.87$\,{\tiny\color{black!65}$\pm 0.90$} & $84.91$\,{\tiny\color{black!65}$\pm 0.39$} & $8.89$\,{\tiny\color{black!65}$\pm 1.57$} & $73.33$\,{\tiny\color{black!65}$\pm 0.25$} & $84.53$\,{\tiny\color{black!65}$\pm 0.32$} & $12.22$\,{\tiny\color{black!65}$\pm 1.57$} \\
    100\% RSP & $75.13$\,{\tiny\color{black!65}$\pm 0.25$} & $\mathbf{85.09}$\,{\tiny\color{black!65}$\pm 0.44$} & $11.11$\,{\tiny\color{black!65}$\pm 1.57$} & $73.80$\,{\tiny\color{black!65}$\pm 0.91$} & $\mathbf{85.06}$\,{\tiny\color{black!65}$\pm 0.57$} & $11.11$\,{\tiny\color{black!65}$\pm 1.57$} & $73.67$\,{\tiny\color{black!65}$\pm 0.52$} & $\mathbf{85.01}$\,{\tiny\color{black!65}$\pm 0.19$} & $8.89$\,{\tiny\color{black!65}$\pm 1.57$} \\
    \midrule
    Token-shuffle & $72.40$\,{\tiny\color{black!65}$\pm 0.99$} & $83.83$\,{\tiny\color{black!65}$\pm 0.38$} & $6.67$\,{\tiny\color{black!65}$\pm 2.72$} & $73.60$\,{\tiny\color{black!65}$\pm 0.49$} & $84.61$\,{\tiny\color{black!65}$\pm 0.06$} & $7.78$\,{\tiny\color{black!65}$\pm 1.57$} & $72.80$\,{\tiny\color{black!65}$\pm 1.56$} & $84.96$\,{\tiny\color{black!65}$\pm 0.29$} & $11.11$\,{\tiny\color{black!65}$\pm 1.57$} \\
    Dim-permute & $71.93$\,{\tiny\color{black!65}$\pm 0.57$} & $83.62$\,{\tiny\color{black!65}$\pm 0.55$} & $5.56$\,{\tiny\color{black!65}$\pm 1.57$} & $74.00$\,{\tiny\color{black!65}$\pm 0.43$} & $84.00$\,{\tiny\color{black!65}$\pm 0.77$} & $\mathbf{12.22}$\,{\tiny\color{black!65}$\pm 1.57$} & $\mathbf{74.80}$\,{\tiny\color{black!65}$\pm 0.16$} & $84.99$\,{\tiny\color{black!65}$\pm 0.16$} & $\mathbf{13.33}$\,{\tiny\color{black!65}$\pm 0.00$} \\
    \bottomrule
  \end{tabular*}}
\end{table}

\paragraph{\rspadd{(i) Is learned content necessary to retain performance?}}
\rspadd{We replace the first half, a randomly selected half, or all trained vectors with independent draws from the RSP distribution in Eq.~\eqref{eq:rsp_def}.} \rspadd{Removing trained vectors} does not produce a consistent loss of accuracy. On the 1.5B model, full replacement raises TTSV's MATH-500 accuracy from $72.33\%$ to $75.13\%$, while LTPO changes from $74.07\%$ to $73.80\%$ and SoftCoT from $72.87\%$ to $73.67\%$. Partial replacement depends on which positions are replaced, as illustrated by TTSV on GSM8K, but retaining more trained vectors does not consistently yield higher accuracy.

\paragraph{\rspadd{(ii) Do trained vectors encode useful information?}}
\rspadd{Since random replacement does not consistently reduce accuracy, we next ask whether the trained vectors actually encode useful information. \textit{Token-shuffle} disrupts their order, while \textit{Dim-permute} independently permutes the coordinates within each vector. These interventions disrupt the organization of the learned content while preserving vector norms. Neither causes a consistent accuracy drop across methods and benchmarks. On MATH-500, dimension permutation changes TTSV's accuracy from $72.33\%$ to $71.93\%$ and LTPO's from $74.07\%$ to $74.00\%$, while increasing SoftCoT's accuracy from $72.87\%$ to $74.80\%$.}

\begin{wraptable}{r}{0.55\textwidth}
  \centering
  \caption{\rspadd{\textbf{Sensitivity to overtrained embeddings.} Results on Qwen2.5-Math-1.5B-Instruct. Accuracy (\%) is reported as mean $\pm$ std over three seeds. Increasing the learning rate by $1000\times$ sharply reduces MATH-500 and GSM8K accuracy.}}
  \label{tab:overtrained_probes}
  \footnotesize\color{black}
  \setlength{\tabcolsep}{2pt}
  \renewcommand{\arraystretch}{0.9}
  \scalebox{0.8594800}{%
  \begin{tabular*}{1.1634942\linewidth}{@{\extracolsep{\fill}}lcccc@{}}
    \toprule
    & \multicolumn{2}{c}{\textbf{TTSV}} & \multicolumn{2}{c}{\textbf{SoftCoT}} \\
    \cmidrule(lr){2-3}\cmidrule(l){4-5}
    \textbf{Benchmark} & Trained & Overtrained & Trained & Overtrained \\
    \midrule
    MATH-500 & $72.33$\,{\tiny\color{black!65}$\pm 1.46$} & $4.20$\,{\tiny\color{black!65}$\pm 2.27$} & $72.87$\,{\tiny\color{black!65}$\pm 1.18$} & $30.60$\,{\tiny\color{black!65}$\pm 22.41$} \\
    GSM8K & $84.08$\,{\tiny\color{black!65}$\pm 1.08$} & $2.60$\,{\tiny\color{black!65}$\pm 3.41$} & $84.84$\,{\tiny\color{black!65}$\pm 1.24$} & $36.01$\,{\tiny\color{black!65}$\pm 25.08$} \\
    AIME24 & $6.67$\,{\tiny\color{black!65}$\pm 0.00$} & $8.89$\,{\tiny\color{black!65}$\pm 4.16$} & $8.89$\,{\tiny\color{black!65}$\pm 1.57$} & $3.33$\,{\tiny\color{black!65}$\pm 4.71$} \\
    \bottomrule
  \end{tabular*}}
  \vspace{-0.6em}
\end{wraptable}

\paragraph{\rspadd{(iii) Is the model simply insensitive to injected vectors?}}
\rspadd{Disrupting the learned content also leaves accuracy largely unchanged. This raises the possibility that the model simply ignores the injected vectors. To test this, we increase the learning rate by $1000\times$ when training TTSV and SoftCoT, keeping the injection token count and position fixed. Table~\ref{tab:overtrained_probes} shows sharp accuracy drops on MATH-500 and GSM8K. TTSV falls to $4.20\%$ and $2.60\%$, and SoftCoT to $30.60\%$ and $36.01\%$, respectively. Taken together, probes (i)--(iii) suggest that, in the evaluated settings, the observed gains arise primarily from embedding injection itself rather than the learned content of the injected vectors.\par}

%% file: sections/5_application.tex
\subsection{Application: Rejection-Sampling SFT with RSP}\label{sec:application}

\begin{wraptable}{R}{0.47\textwidth}
\vspace{-\intextsep}
\centering
\caption{\rspadd{\textbf{Correct-solution collection on hard problems.} Results on the fixed hard subset using Qwen2.5-Math-7B-Instruct, with eight rollouts per problem. Percentages are reported as mean $\pm$ std over three collection seeds. Coverage is the fraction of problems with at least one correct rollout.}}
\label{tab:sft_collection}
\footnotesize
\setlength{\tabcolsep}{2.25pt}
\renewcommand{\arraystretch}{0.9}
\begin{tabular}{@{}lcc@{}}
\toprule
\textbf{Method} & \textbf{Problem coverage (\%)} & \textbf{Rollout accuracy (\%)} \\
\midrule
Baseline & $23.75$\,{\tiny$\pm 0.10$} & $5.54$\,{\tiny$\pm 0.03$} \\
\rowcolor{blue!10}
RSP & $49.68$\,{\tiny$\pm 0.13$} & $17.45$\,{\tiny$\pm 0.05$} \\
\bottomrule
\end{tabular}
\end{wraptable}

\rspadd{Section~\ref{sec:diversity} showed that independent RSP resampling expands trajectory diversity and improves Pass@$N$, including on the challenging AIME24 benchmark. We now ask whether this advantage can improve the collection of correct solutions for supervised fine-tuning (SFT), focusing on problems for which the initial model has a low solution success rate. We study rejection-sampling SFT~\citep{singh2024beyond} using Qwen2.5-Math-7B-Instruct. We define a hard subset of OpenR1-Math-46k~\citep{yan2025learning} as the 32{,}737 problems that this model solves at most once in four rollouts without RSP. For each hard problem, Baseline and RSP each generate eight rollouts at $\tau=1.0$. We retain only correct solutions and fine-tune the same initial model separately on each collected dataset for one epoch. We evaluate the model without SFT and the two models fine-tuned on solutions collected by Baseline and RSP on Minerva~\citep{lewkowycz2022minerva}, GaoKao~\citep{liao2024mario}, OlympiadBench~\citep{he2024olympiad}, AIME24, and AMC23. All evaluations are performed without RSP. Results are averaged over three data-collection seeds (training and seed details in Appendix~\ref{app:sft_setup}).\par}

\rspadd{Table~\ref{tab:sft_collection} shows that RSP collects at least one correct solution for $49.68\%$ of the hard problems (problem coverage), compared with $23.75\%$ for Baseline, approximately $2.1\times$ the coverage under the same rollout budget. Per-rollout accuracy also increases from $5.54\%$ to $17.45\%$. Thus, RSP provides a larger pool of correct solutions spanning more hard problems for subsequent SFT.}

\begin{table}[!ht]
\centering
\caption{\rspadd{\textbf{SFT with self-collected correct solutions.} Evaluation without RSP on Qwen2.5-Math-7B-Instruct. Accuracy (\%) is reported as mean $\pm$ std over three collection seeds for the SFT conditions. AIME24 and AMC23 use Avg@32. Avg. denotes the mean across five benchmarks.}}
\label{tab:sft_results}
\footnotesize
\setlength{\tabcolsep}{4pt}
\renewcommand{\arraystretch}{0.9}
\begin{tabular}{lcccccc}
\toprule
\textbf{Method} & \textbf{Minerva} & \textbf{GaoKao} & \textbf{Olympiad} & \textbf{AIME24} & \textbf{AMC23} & \textbf{Avg} \\
\midrule
No SFT & $39.7$ & $66.8$ & $40.4$ & $6.7$ & $47.7$ & $40.26$ \\
Baseline-SFT & $42.7$\,{\tiny$\pm 1.6$} & $67.0$\,{\tiny$\pm 0.5$} & $39.4$\,{\tiny$\pm 1.1$} & $9.2$\,{\tiny$\pm 0.6$} & $52.2$\,{\tiny$\pm 1.0$} & $42.11$\,{\tiny$\pm 0.54$} \\
\rowcolor{blue!10}
RSP-SFT & $43.2$\,{\tiny$\pm 0.8$} & $69.4$\,{\tiny$\pm 1.2$} & $40.7$\,{\tiny$\pm 0.8$} & $10.0$\,{\tiny$\pm 0.1$} & $55.5$\,{\tiny$\pm 1.8$} & $43.77$\,{\tiny$\pm 0.54$} \\
\bottomrule
\end{tabular}
\end{table}

\rspadd{The collection advantage also transfers to subsequent training (Table~\ref{tab:sft_results}). RSP-SFT reaches a five-benchmark average of $43.77\%$, compared with $42.11\%$ for Baseline-SFT and $40.26\%$ without SFT. The $1.66$\,pp gain over Baseline-SFT suggests that RSP can improve training-data collection as well as inference by obtaining correct solutions for more hard problems. This comparison includes differences in both the amount and composition of the collected data. Additional DAPO results are reported in Appendix~\ref{app:dapo}.}

%% file: sections/6_conclusion.tex

\section{Conclusion}\label{sec:conclusion}
\enlargethispage*{2\baselineskip}

\rspadd{We studied Random Soft Prompts (RSPs) to distinguish the contribution of learned content from the effect of embedding injection itself. Across the evaluated settings, untrained random vectors reach the performance range of optimized \textit{soft prompts}, while replacing or disrupting trained vectors does not consistently reduce accuracy. These results suggest that embedding injection itself contributes to the gains observed with trained \textit{soft prompts}.}

\rspadd{Our analysis connects RSP-induced trajectory diversity to perturbations of early decoding decisions and the subsequent attenuation of RSP attention. We express this attention mass in terms of token counts and the evolving attention-logit gap. Empirically, RSP exhibits early output-distribution patterns similar to optimized \textit{soft prompts}, increases trajectory diversity, and improves Pass@$N$ under independent resampling. Beyond inference, RSP collects correct solutions for more hard problems and improves subsequent SFT. Together, these findings identify embedding injection as a source of reasoning diversity that can benefit both inference and training-data collection.}

%% file: appendix/appendix.tex

\clearpage
\section{Experimental Setup and Full Accuracy Breakdown}\label{app:setup}\label{app:ablation}

\rspadd{The math accuracy experiments in Tables~\ref{tab:prior_comparison} and~\ref{tab:performance} are conducted on a single NVIDIA RTX A6000 48GB GPU with greedy decoding}, a maximum generation length of 3{,}072 tokens for MATH-500 and GSM8K, and 4{,}096 tokens for AIME24. Batch size is fixed per model (16 for LLaMA-3.1-8B-Instruct and Qwen2.5-Math-7B variants, 32 for Qwen2.5-Math-1.5B variants). The RSP token count $L$ is selected per (model, benchmark) pair from $\{10, 15, 20\}$ (Table~\ref{tab:token_count}). \rspadd{The single-seed results across \textit{prefix}, \textit{infix}, and \textit{suffix} injection positions are reported in Table~\ref{tab:performance_full}.}

\begin{table}[H]
\centering
\caption{Accuracy (\%) across model configurations, injection positions, and benchmarks. Values in parentheses indicate the difference from the baseline. \textbf{Bold} denotes the best result and \underline{underline} denotes the second best for each model--benchmark pair.}
\label{tab:performance_full}
\footnotesize
\renewcommand{\arraystretch}{0.9}
\scalebox{0.8594800}{%
\begin{tabular}{ll ccc}
\toprule
\textbf{Model} & \textbf{Mode} & \textbf{MATH-500} & \textbf{GSM8K} & \textbf{AIME24} \\
\midrule
\multicolumn{5}{l}{\textit{Instruct Models}} \\
\midrule
\multirow{4}{*}{LLaMA-3.1-8B-Instruct} & Baseline & \textbf{52.20} & 85.75 & \underline{6.67} \\
 & \textit{Prefix} & 45.00 {\scriptsize($-$7.2)} & 76.35 {\scriptsize($-$9.4)} & 3.33 {\scriptsize($-$3.3)} \\
 & \textit{Infix} & \underline{49.40} {\scriptsize($-$2.8)} & \underline{85.97} {\scriptsize(+0.2)} & \textbf{10.00} {\scriptsize(+3.3)} \\
 & \textit{Suffix} & \underline{49.40} {\scriptsize($-$2.8)} & \textbf{86.73} {\scriptsize(+1.0)} & \textbf{10.00} {\scriptsize(+3.3)} \\
\midrule
\multirow{4}{*}{Qwen2.5-Math-7B-Instruct} & Baseline & 83.20 & 95.45 & \underline{13.33} \\
 & \textit{Prefix} & 81.40 {\scriptsize($-$1.8)} & 94.92 {\scriptsize($-$0.5)} & \underline{13.33} {\scriptsize(+0.0)} \\
 & \textit{Infix} & \textbf{84.20} {\scriptsize(+1.0)} & \underline{95.60} {\scriptsize(+0.2)} & \textbf{16.67} {\scriptsize(+3.3)} \\
 & \textit{Suffix} & \underline{83.40} {\scriptsize(+0.2)} & \textbf{95.68} {\scriptsize(+0.2)} & \textbf{16.67} {\scriptsize(+3.3)} \\
\midrule
\multirow{4}{*}{Qwen2.5-Math-1.5B-Instruct} & Baseline & 73.00 & \underline{85.29} & 10.00 \\
 & \textit{Prefix} & \underline{74.80} {\scriptsize(+1.8)} & \underline{85.29} {\scriptsize(+0.0)} & \underline{13.33} {\scriptsize(+3.3)} \\
 & \textit{Infix} & \textbf{75.20} {\scriptsize(+2.2)} & \textbf{85.75} {\scriptsize(+0.5)} & 6.67 {\scriptsize($-$3.3)} \\
 & \textit{Suffix} & 74.20 {\scriptsize(+1.2)} & 84.69 {\scriptsize($-$0.6)} & \textbf{20.00} {\scriptsize(+10.0)} \\
\midrule
\multicolumn{5}{l}{\textit{Base Models (Raw Text)}} \\
\midrule
\multirow{4}{*}{Qwen2.5-Math-7B} & Baseline & \textbf{72.20} & \underline{84.15} & 6.67 \\
 & \textit{Prefix} & \underline{71.60} {\scriptsize($-$0.6)} & \textbf{84.31} {\scriptsize(+0.2)} & \textbf{16.67} {\scriptsize(+10.0)} \\
 & \textit{Infix} & 63.20 {\scriptsize($-$9.0)} & 64.29 {\scriptsize($-$19.9)} & \textbf{16.67} {\scriptsize(+10.0)} \\
 & \textit{Suffix} & 55.60 {\scriptsize($-$16.6)} & 54.13 {\scriptsize($-$30.0)} & \underline{13.33} {\scriptsize(+6.7)} \\
\midrule
\multirow{4}{*}{Qwen2.5-Math-1.5B} & Baseline & 61.40 & \textbf{77.86} & \textbf{16.67} \\
 & \textit{Prefix} & \textbf{64.40} {\scriptsize(+3.0)} & \underline{74.30} {\scriptsize($-$3.6)} & \underline{10.00} {\scriptsize($-$6.7)} \\
 & \textit{Infix} & \underline{63.20} {\scriptsize(+1.8)} & 73.92 {\scriptsize($-$3.9)} & \underline{10.00} {\scriptsize($-$6.7)} \\
 & \textit{Suffix} & 54.60 {\scriptsize($-$6.8)} & 58.61 {\scriptsize($-$19.3)} & 3.33 {\scriptsize($-$13.3)} \\
\midrule
\multicolumn{5}{l}{\textit{Base Models + ChatML (Format Mismatch)}} \\
\midrule
\multirow{4}{*}{Qwen2.5-Math-7B} & Baseline & 52.20 & 58.30 & \textbf{23.33} \\
 & \textit{Prefix} & 59.20 {\scriptsize(+7.0)} & 56.41 {\scriptsize($-$1.9)} & \underline{3.33} {\scriptsize($-$20.0)} \\
 & \textit{Infix} & \underline{62.00} {\scriptsize(+9.8)} & \underline{69.07} {\scriptsize(+10.8)} & \textbf{23.33} {\scriptsize(+0.0)} \\
 & \textit{Suffix} & \textbf{70.40} {\scriptsize(+18.2)} & \textbf{75.97} {\scriptsize(+17.7)} & \textbf{23.33} {\scriptsize(+0.0)} \\
\midrule
\multirow{4}{*}{Qwen2.5-Math-1.5B} & Baseline & 34.40 & 37.45 & \underline{13.33} \\
 & \textit{Prefix} & \textbf{63.40} {\scriptsize(+29.0)} & \textbf{67.02} {\scriptsize(+29.6)} & \textbf{20.00} {\scriptsize(+6.7)} \\
 & \textit{Infix} & 39.20 {\scriptsize(+4.8)} & 50.42 {\scriptsize(+13.0)} & 6.67 {\scriptsize($-$6.7)} \\
 & \textit{Suffix} & \underline{51.00} {\scriptsize(+16.6)} & \underline{50.64} {\scriptsize(+13.2)} & \underline{13.33} {\scriptsize(+0.0)} \\
\bottomrule
\end{tabular}}
\end{table}

\begin{table}[H]
\centering
\caption{Number of RSP tokens ($L$) used for each model and benchmark.}
\label{tab:token_count}
\footnotesize
\renewcommand{\arraystretch}{0.9}
\scalebox{0.8594800}{%
\begin{tabular}{lccc}
\toprule
\textbf{Model} & \textbf{MATH-500} & \textbf{GSM8K} & \textbf{AIME24} \\
\midrule
LLaMA-3.1-8B-Instruct & 15 & 20 & 15 \\
Qwen2.5-Math-7B-Instruct & 20 & 20 & 20 \\
Qwen2.5-Math-1.5B-Instruct & 20 & 10 & 20 \\
Qwen2.5-Math-7B & 10 & 10 & 20 \\
Qwen2.5-Math-1.5B & 10 & 10 & 20 \\
Qwen2.5-Math-1.5B (ChatML) & 20 & 20 & 10 \\
Qwen2.5-Math-7B (ChatML) & 20 & 20 & 20 \\
\bottomrule
\end{tabular}}
\end{table}

\subsection{RSP Injection Positions and Prompt Templates}\label{app:injection_positions}

Below we show the prompt structure for each injection position across all model types. \textcolor{blue!70!black}{Blue text} indicates RSP embeddings, and \textcolor{purple}{purple text} indicates special tokens.

\subsubsection{\textit{Prefix}}

RSP embeddings are concatenated before the prompt embeddings. No text modification is applied.

\paragraph{ChatML (Qwen Instruct / Base + ChatML).}
\begin{tcolorbox}[colback=gray!5, colframe=black!60, arc=3pt, boxrule=0.5pt, fontupper=\footnotesize\ttfamily, breakable]
\textcolor{blue!70!black}{[Random Embeddings ($L$ tokens)]}\\
\textcolor{purple}{<|im\_start|>}system\\
Please reason step by step, and put your final answer within \textbackslash boxed\{\}.\textcolor{purple}{<|im\_end|>}\\
\textcolor{purple}{<|im\_start|>}user\\
\{question\}\textcolor{purple}{<|im\_end|>}\\
\textcolor{purple}{<|im\_start|>}assistant
\end{tcolorbox}

\paragraph{LLaMA Chat Template.}
\begin{tcolorbox}[colback=gray!5, colframe=black!60, arc=3pt, boxrule=0.5pt, fontupper=\footnotesize\ttfamily, breakable]
\textcolor{blue!70!black}{[Random Embeddings ($L$ tokens)]}\\
\textcolor{purple}{<|begin\_of\_text|>}\\
\textcolor{purple}{<|start\_header\_id|>}system\textcolor{purple}{<|end\_header\_id|>}\\
Cutting Knowledge Date: December 2023\\
Today Date: 26 Jul 2024\\
Please reason step by step, and put your final answer within \textbackslash boxed\{\}.\textcolor{purple}{<|eot\_id|>}\\
\textcolor{purple}{<|start\_header\_id|>}user\textcolor{purple}{<|end\_header\_id|>}\\
\{question\}\textcolor{purple}{<|eot\_id|>}\\
\textcolor{purple}{<|start\_header\_id|>}assistant\textcolor{purple}{<|end\_header\_id|>}
\end{tcolorbox}

\paragraph{Raw Text (Qwen Base).}
\begin{tcolorbox}[colback=gray!5, colframe=black!60, arc=3pt, boxrule=0.5pt, fontupper=\footnotesize\ttfamily, breakable]
\textcolor{blue!70!black}{[Random Embeddings ($L$ tokens)]}\\
\{question\}\\
Please reason step by step, and put your final answer within \textbackslash boxed\{\}.
\end{tcolorbox}

\subsubsection{\textit{Infix}}

A description text and special tokens are inserted into the prompt. The special token positions are then replaced with RSP embeddings at the embedding level.

\paragraph{ChatML (Qwen Instruct / Base + ChatML).}
\begin{tcolorbox}[colback=gray!5, colframe=black!60, arc=3pt, boxrule=0.5pt, fontupper=\footnotesize\ttfamily, breakable]
\textcolor{purple}{<|im\_start|>}system\\
Please reason step by step, and put your final answer within \textbackslash boxed\{\}.\textcolor{purple}{<|im\_end|>}\\
\textcolor{purple}{<|im\_start|>}user\\
\{question\}\\[0.3em]
\textcolor{gray}{There are \{$L$\} special tokens that contain compressed latent reasoning information that might be useful for your reasoning. If these tokens are useful for your case, you can use them as reference. If these tokens are not useful for your case, you can ignore them and focus back to solving the problem.}\\[0.3em]
Here are the \{$L$\} special tokens: \textcolor{blue!70!black}{<special\_token\_1>...<special\_token\_$L$>}\\
\textcolor{purple}{<|im\_end|>}\\
\textcolor{purple}{<|im\_start|>}assistant
\end{tcolorbox}

\paragraph{LLaMA Chat Template.}
\begin{tcolorbox}[colback=gray!5, colframe=black!60, arc=3pt, boxrule=0.5pt, fontupper=\footnotesize\ttfamily, breakable]
\textcolor{purple}{<|begin\_of\_text|>}\\
\textcolor{purple}{<|start\_header\_id|>}system\textcolor{purple}{<|end\_header\_id|>}\\
Cutting Knowledge Date: December 2023\\
Today Date: 26 Jul 2024\\
Please reason step by step, and put your final answer within \textbackslash boxed\{\}.\textcolor{purple}{<|eot\_id|>}\\
\textcolor{purple}{<|start\_header\_id|>}user\textcolor{purple}{<|end\_header\_id|>}\\
\{question\}\\[0.3em]
\textcolor{gray}{There are \{$L$\} special tokens that contain compressed latent reasoning information that might be useful for your reasoning. If these tokens are useful for your case, you can use them as reference. If these tokens are not useful for your case, you can ignore them and focus back to solving the problem.}\\[0.3em]
Here are the \{$L$\} special tokens:\\
\textcolor{blue!70!black}{<reserved\_special\_token\_0>...}\\
\textcolor{blue!70!black}{<reserved\_special\_token\_$L$>}\\
\textcolor{purple}{<|eot\_id|>}\\
\textcolor{purple}{<|start\_header\_id|>}assistant\textcolor{purple}{<|end\_header\_id|>}
\end{tcolorbox}

\paragraph{Raw Text (Qwen Base).}
\begin{tcolorbox}[colback=gray!5, colframe=black!60, arc=3pt, boxrule=0.5pt, fontupper=\footnotesize\ttfamily, breakable]
\{question\}\\[0.3em]
\textcolor{gray}{There are \{$L$\} special tokens that contain compressed latent reasoning information that might be useful for your reasoning. If these tokens are useful for your case, you can use them as reference. If these tokens are not useful for your case, you can ignore them and focus back to solving the problem.}\\[0.3em]
Here are the \{$L$\} special tokens: \textcolor{blue!70!black}{<special\_token\_1>...<special\_token\_$L$>}\\[0.3em]
Please reason step by step, and put your final answer within \textbackslash boxed\{\}.
\end{tcolorbox}

\subsubsection{\textit{Suffix}}

For chat-templated models, RSP embeddings are inserted immediately before the end-of-turn token. For raw text models, RSP embeddings are concatenated at the end of the prompt.

\paragraph{ChatML (Qwen Instruct / Base + ChatML).}
\begin{tcolorbox}[colback=gray!5, colframe=black!60, arc=3pt, boxrule=0.5pt, fontupper=\footnotesize\ttfamily, breakable]
\textcolor{purple}{<|im\_start|>}system\\
Please reason step by step, and put your final answer within \textbackslash boxed\{\}.\textcolor{purple}{<|im\_end|>}\\
\textcolor{purple}{<|im\_start|>}user\\
\{question\}\textcolor{blue!70!black}{[Random Embeddings ($L$ tokens)]}\textcolor{purple}{<|im\_end|>}\\
\textcolor{purple}{<|im\_start|>}assistant
\end{tcolorbox}

\paragraph{LLaMA Chat Template.}
\begin{tcolorbox}[colback=gray!5, colframe=black!60, arc=3pt, boxrule=0.5pt, fontupper=\footnotesize\ttfamily, breakable]
\textcolor{purple}{<|begin\_of\_text|>}\\
\textcolor{purple}{<|start\_header\_id|>}system\textcolor{purple}{<|end\_header\_id|>}\\
Cutting Knowledge Date: December 2023\\
Today Date: 26 Jul 2024\\
Please reason step by step, and put your final answer within \textbackslash boxed\{\}.\textcolor{purple}{<|eot\_id|>}\\
\textcolor{purple}{<|start\_header\_id|>}user\textcolor{purple}{<|end\_header\_id|>}\\
\{question\}\textcolor{blue!70!black}{[Random Embeddings ($L$ tokens)]}\textcolor{purple}{<|eot\_id|>}\\
\textcolor{purple}{<|start\_header\_id|>}assistant\textcolor{purple}{<|end\_header\_id|>}
\end{tcolorbox}

\paragraph{Raw Text (Qwen Base).}
\begin{tcolorbox}[colback=gray!5, colframe=black!60, arc=3pt, boxrule=0.5pt, fontupper=\footnotesize\ttfamily, breakable]
\{question\}\\
Please reason step by step, and put your final answer within \textbackslash boxed\{\}.\textcolor{blue!70!black}{[Random Embeddings ($L$ tokens)]}
\end{tcolorbox}

\subsection{Answer Extraction and Grading}\label{app:extraction}

The final answer is extracted from the model output through a fallback chain. Each step is attempted in order. If extraction fails, the next step is tried.

\begin{tcolorbox}[colback=gray!5, colframe=black!60, arc=3pt, boxrule=0.5pt, title=\textbf{Answer Extraction Fallback Chain}, fonttitle=\small, coltitle=black, colbacktitle=gray!20]
\small
\begin{enumerate}[leftmargin=*, nosep]
  \item \texttt{"final answer is \$...\$. I hope"} pattern (Minerva format)
  \item \texttt{\textbackslash boxed\{...\}}: last non-empty box is used, and empty \texttt{\textbackslash boxed\{\}} are skipped
  \item Text following \texttt{"the answer is"}
  \item Text following \texttt{"final answer is"}
  \item Last number in the output (regex fallback)
\end{enumerate}
\end{tcolorbox}

\vspace{0.5em}
Extracted answers are normalized by removing \texttt{\$}, \texttt{\textbackslash left}, \texttt{\textbackslash right}, and unit strings. Grading is performed via symbolic equivalence checking: exact string match, numerical comparison (\texttt{math.isclose}, \texttt{rel\_tol=1e-4}), and SymPy symbolic simplification.

\subsection{Reproduced Prior Work Hyperparameters}\label{app:prior_hp}

All prior methods are reproduced on Qwen2.5-Math-1.5B-Instruct and Qwen2.5-Math-7B-Instruct, evaluated with greedy decoding (temperature $= 0$) and our unified answer extraction pipeline.

\paragraph{TTSV.} \textit{Prefix} length 20, trained for 20 epochs with AdamW optimizer, batch size 2, gradient accumulation 8 steps, and linear learning rate schedule. Learning rate is 1e-3 for Qwen-1.5B and 1e-5 for Qwen-7B. \rspadd{We vary the random seed controlling \textit{prefix} initialization and training-data shuffling, using seeds 15--19 for Table~\ref{tab:prior_comparison}. The content probes in Section~\ref{sec:content_probes} reuse the prefixes trained with seeds 15--17.}

\paragraph{LTPO.} Table~\ref{tab:hp_ltpo} reports the per-benchmark hyperparameters. \rspadd{We vary the random seed controlling Gaussian perturbations during per-problem optimization, using seeds 42--46 for Table~\ref{tab:prior_comparison}. The content probes in Section~\ref{sec:content_probes} reuse the prompt vectors optimized with seeds 42--44.}

\begin{table}[H]
\centering
\caption{LTPO hyperparameters.}
\label{tab:hp_ltpo}
\footnotesize
\renewcommand{\arraystretch}{0.9}
\scalebox{0.8594800}{%
\begin{tabular}{lccc}
\toprule
\textbf{Parameter} & \textbf{GSM8K} & \textbf{MATH-500} & \textbf{AIME24} \\
\midrule
tokens & 8 & 8 & 8 \\
steps & 20 & 20 & 20 \\
top\_k & 10 & 10 & 10 \\
sigma & 20 & 20 & 5 \\
sigma\_decay & 0.95 & 0.95 & 0.95 \\
lr & 1e-2 & 5e-3 & 5e-2 \\
\bottomrule
\end{tabular}}
\end{table}

\paragraph{SoftCoT.} Projection module trained for 10 epochs. Thought tokens: 32 during training, 4 during evaluation. Batch size 4 for GSM8K, 1 for MATH with gradient accumulation 4. The small (assistant) model is Qwen2.5-Math-1.5B-Instruct in both configurations, while the large model is Qwen2.5-Math-1.5B-Instruct (learning rate 2e-5) or Qwen2.5-Math-7B-Instruct (learning rate 5e-6). For MATH-500 and AIME24 evaluation, we use the projection module trained on GSM8K, as it yields higher accuracy than the one trained on MATH. \rspadd{We vary the random seed controlling projection initialization and training-data shuffling, using seeds 42--46 for Table~\ref{tab:prior_comparison}. The content probes in Section~\ref{sec:content_probes} reuse the projection modules trained with seeds 42--44.}

\input{appendix/sft_setup}

\clearpage
\subsection{Full Entropy Curves}\label{app:entropy_full}

Figure~\ref{fig:app_full_entropy} shows the full normalized (0--100\%) generation trajectories for entropy, top-1 probability, and varentropy, and Table~\ref{tab:entropy_full} reports the corresponding scalar statistics including overall means, first-5\% means, and the average number of generated tokens per problem. All methods converge to similar levels after the initial generation stage, confirming that the distributional changes induced by RSP are localized to the early reasoning phase.

\begin{table}[H]
\centering
\caption{Per-step statistics for Qwen2.5-Math-1.5B-Instruct on MATH-500 across the full generation and the first 5\% of generation steps, together with the average number of generated tokens per problem. Top-1 Prob (overall) is reported as a weighted average over the first-10\%/middle/last-10\% segment means with weights $0.1/0.8/0.1$. Extension of Figure~\ref{fig:entropy_first5} (main text) and Figure~\ref{fig:app_full_entropy}.}
\label{tab:entropy_full}
\footnotesize
\setlength{\tabcolsep}{4pt}
\renewcommand{\arraystretch}{0.9}
\scalebox{0.8594800}{%
\begin{tabular}{lccccccc}
\toprule
\textbf{Metric} & \textbf{Baseline} & \cellcolor{blue!10}\textbf{\textit{Prefix}} & \cellcolor{blue!10}\textbf{\textit{Infix}} & \cellcolor{blue!10}\textbf{\textit{Suffix}} & \textbf{LTPO} & \textbf{SoftCoT} & \textbf{TTSV} \\
\midrule
Tokens (mean)            & 575.7  & \cellcolor{blue!10}558.4  & \cellcolor{blue!10}549.9  & \cellcolor{blue!10}573.3  & 592.3  & 594.4  & 577.4  \\
\midrule
Entropy (first 5\%)      & 0.1332 & \cellcolor{blue!10}0.1366 & \cellcolor{blue!10}0.1402 & \cellcolor{blue!10}0.1941 & 0.1662 & 0.4218 & 0.1359 \\
Entropy (overall)        & 0.0871 & \cellcolor{blue!10}0.0881 & \cellcolor{blue!10}0.0899 & \cellcolor{blue!10}0.1049 & 0.0909 & 0.1059 & 0.0864 \\
\midrule
Top-1 Prob (first 5\%)   & 0.9535 & \cellcolor{blue!10}0.9523 & \cellcolor{blue!10}0.9525 & \cellcolor{blue!10}0.9373 & 0.9437 & 0.9156 & 0.9534 \\
Top-1 Prob (overall)     & 0.9684 & \cellcolor{blue!10}0.9681 & \cellcolor{blue!10}0.9678 & \cellcolor{blue!10}0.9647 & 0.9683 & 0.9651 & 0.9685 \\
\midrule
Varentropy (first 5\%)   & 0.2181 & \cellcolor{blue!10}0.2207 & \cellcolor{blue!10}0.2463 & \cellcolor{blue!10}0.4010 & 0.2953 & 2.2234 & 0.2174 \\
Varentropy (overall)     & 0.1353 & \cellcolor{blue!10}0.1375 & \cellcolor{blue!10}0.1422 & \cellcolor{blue!10}0.2038 & 0.1452 & 0.2404 & 0.1361 \\
\bottomrule
\end{tabular}}
\end{table}

\begin{figure}[H]
  \centering
  \includegraphics[width=0.65\textwidth]{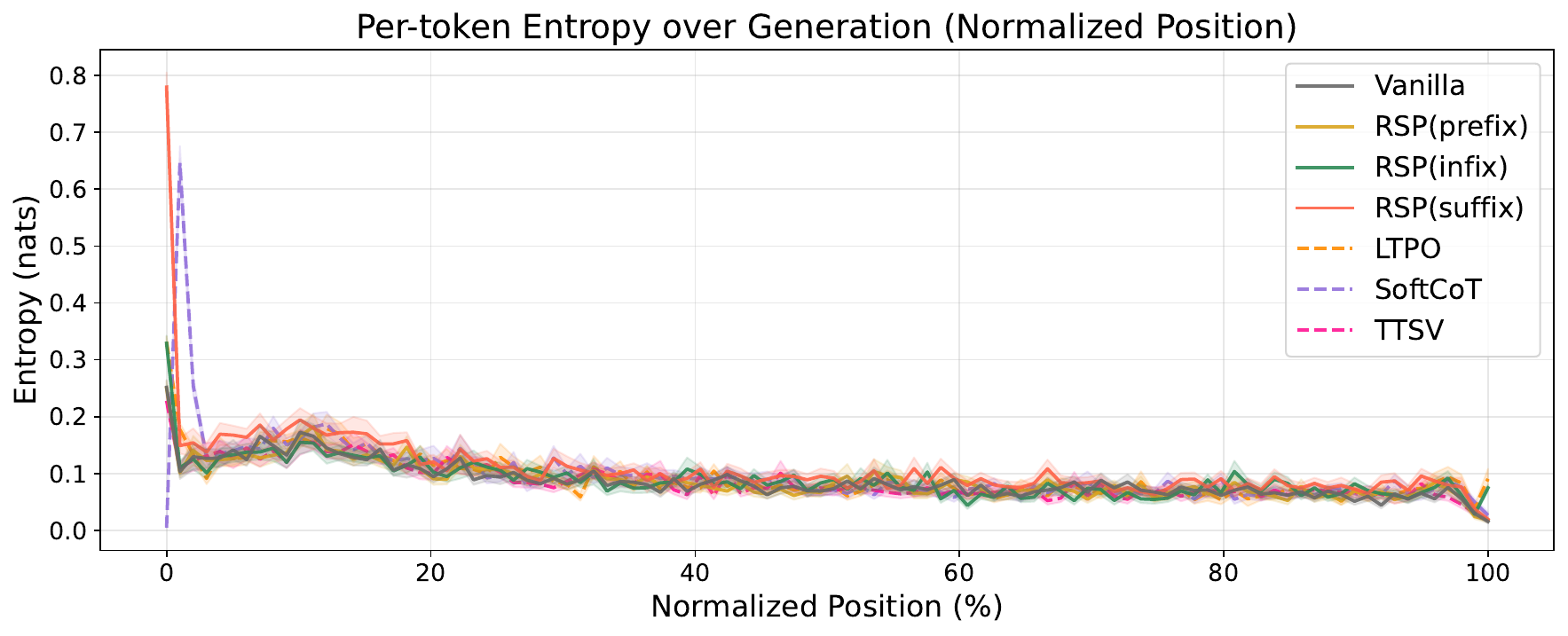}\\[0.4em]
  \includegraphics[width=0.65\textwidth]{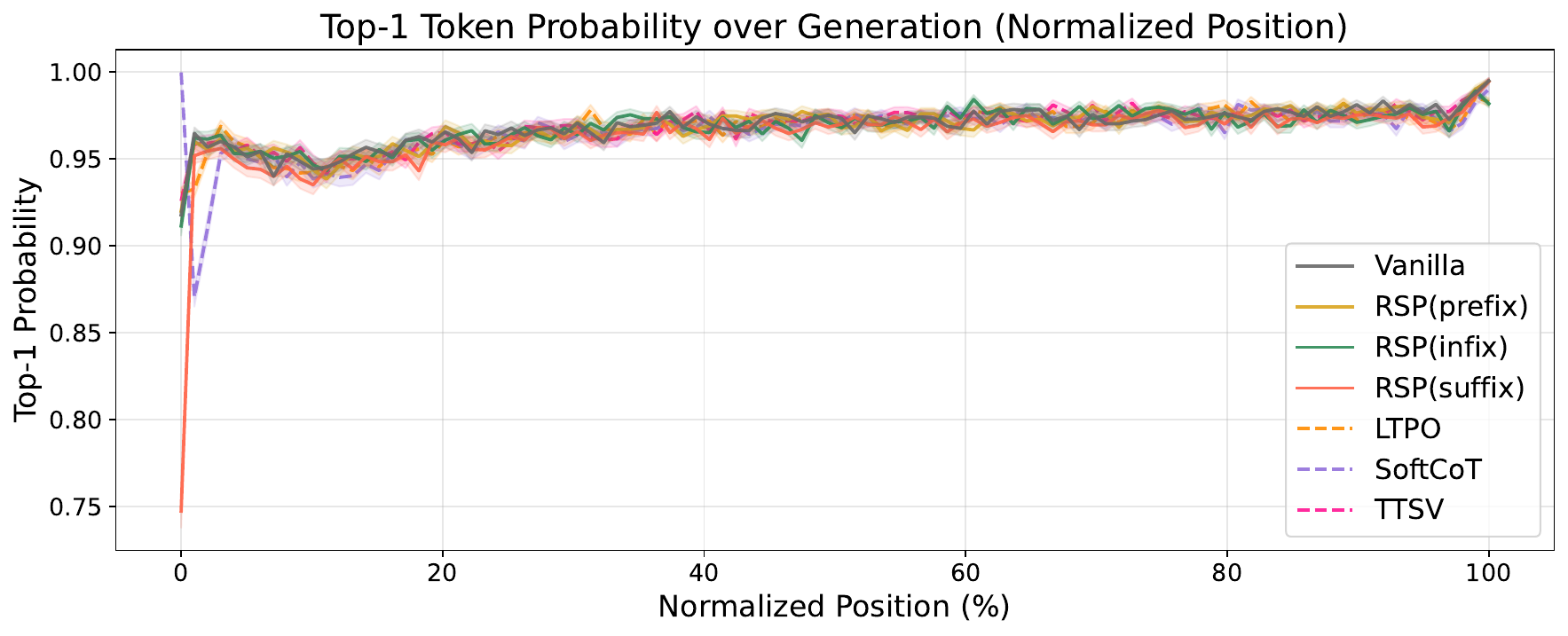}\\[0.4em]
  \includegraphics[width=0.65\textwidth]{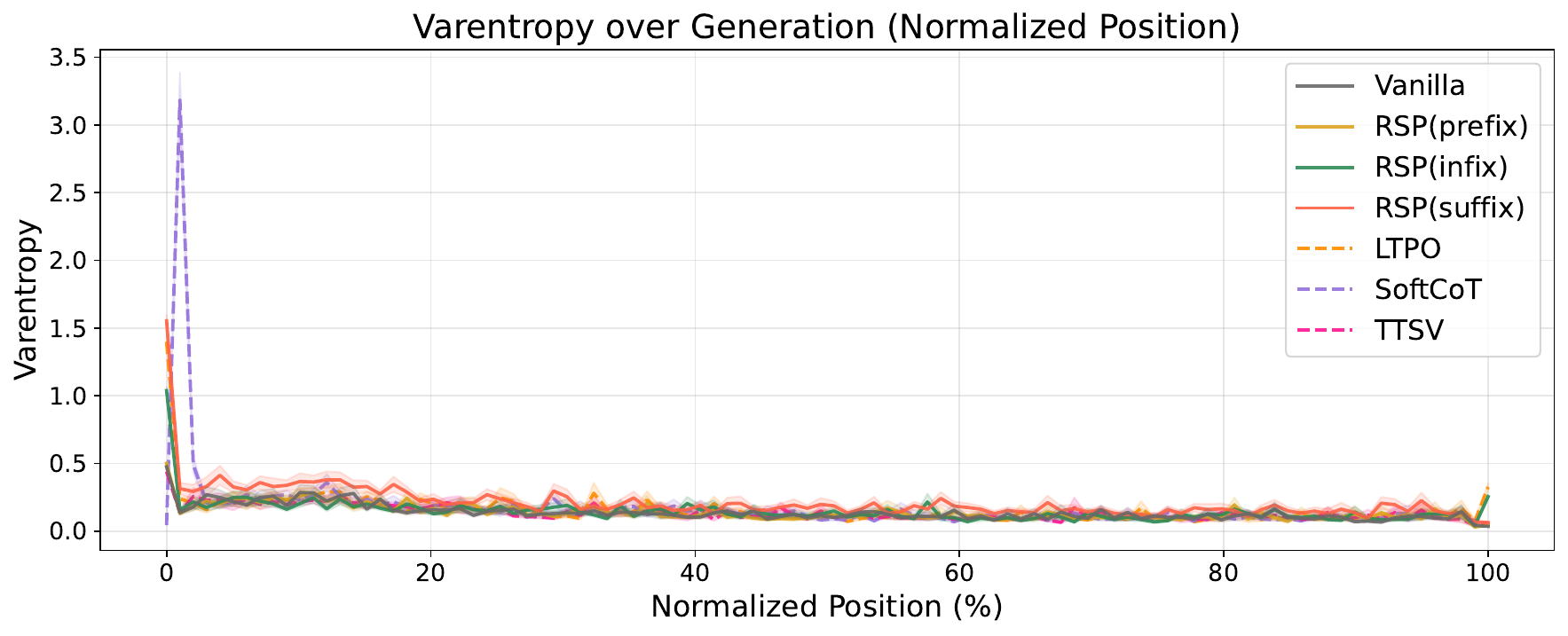}
  \caption{Mean entropy (top), top-1 probability (middle), and varentropy (bottom) over the full normalized generation trajectory (0--100\%). Averaged over MATH-500, shading indicates $\pm$1 SEM. Solid lines: RSP variants and the baseline. Dashed lines: prior \textit{soft prompt} methods. All methods converge after the initial stage.}
  \label{fig:app_full_entropy}
\end{figure}

\input{appendix/content_probes_7b}

\input{appendix/law_scale_ablation}

\section{Token-Level Distribution Metrics}\label{app:token_metrics}

This appendix provides formal definitions for the metrics used in Section~\ref{sec:diversity}. Numerical results are reported in Table~\ref{tab:rsp_vs_temp} of the main text. We extract first-token logits via a single forward pass (no autoregressive generation) and store the top-100 token ids and logit values per problem. Temperature conditions reuse the baseline logits scaled by $1/\tau$, requiring no separate inference. RSP draws 16 independent seeds per problem (MATH-500), and reported metrics are the mean across the 16 seeds. Pass@1 accuracy is computed over 16 samples per problem, with the baseline at $\tau{=}1.0$ reaching $73.92 \pm 0.78\%$. Let $P_{\text{base}}$ denote the baseline next-token distribution at $\tau{=}1.0$ and $P_{\text{target}}$ the candidate distribution being compared (e.g., baselines at $\tau{=}2.0, 3.0$ or RSP at $\tau{=}1.0$). All metrics are computed analytically from saved logits at the first generation step.

\paragraph{Spearman rank correlation.}
Spearman's $\rho$ is the Pearson correlation coefficient between the token ranks of two distributions:
\begin{equation}
\rho \;=\; \mathrm{Pearson}\!\left(\mathrm{rank}(P_{\text{target}}),\; \mathrm{rank}(P_{\text{base}})\right).
\end{equation}
Because temperature scaling $P^{(\tau)} \propto P^{1/\tau}$ is a strictly monotone transform, the rank order of tokens is preserved exactly, so $\rho = 1$ for any $\tau > 0$ as a mathematical identity. Any value $\rho < 1$ therefore indicates rank reorganization beyond what temperature scaling can produce.

\paragraph{Mass outside top-$K$.}
Let $\mathrm{TopK}(P_{\text{base}})$ be the set of $K$ tokens with the highest probability under $P_{\text{base}}$. The probability mass that the candidate distribution places \emph{outside} this set is
\begin{equation}
\mathrm{MassOutside}_K \;=\; 1 - \sum_{t \in \mathrm{TopK}(P_{\text{base}})} P_{\text{target}}(t).
\end{equation}
This directly measures support expansion: how much probability the candidate assigns to tokens that are not in the baseline's preferred set. Reported values use $K = 10$.

\paragraph{Jensen--Shannon divergence.}
The symmetrized, bounded version of Kullback--Leibler divergence:
\begin{equation}
\mathrm{JS}(P, Q) \;=\; \tfrac{1}{2}\,\mathrm{KL}(P \,\|\, M) + \tfrac{1}{2}\,\mathrm{KL}(Q \,\|\, M), \qquad M = \tfrac{1}{2}(P + Q).
\end{equation}
We use base-2 logarithms so that $\mathrm{JS} \in [0, 1]$. Tokens absent from the stored top-100 are assigned a sentinel logit ($-10^9$) before softmax, which is negligible for our setting.

\input{appendix/temperature_comparison}

\section{Qualitative Analysis: How RSP Reshapes Reasoning Trajectories}\label{app:qualitative}

This appendix complements the outcome-level Pass@$N$ results of \S\ref{sec:diversity} with a qualitative case study on AIME24, using the same $N{=}16$ sampling budget. We compare two of the four conditions defined in \S\ref{sec:diversity}: \emph{Baseline} (condition (i), $\tau{=}1.0$) and \emph{RSP (indep.\ seed)} (condition (iv), $\tau{=}1.0$, \textit{suffix}, $L{=}20$). Setting: Qwen2.5-Math-1.5B-Instruct, 30 AIME24 problems, $16$ samples per problem per condition.

\subsection{Aggregate Pattern Frequencies}\label{app:qualitative_stats}

Table~\ref{tab:qualitative_patterns} tabulates how RSP shifts coverage relative to Baseline across the 30 AIME24 problems.

\begin{table}[H]
\centering
\caption{Coverage patterns across 30 AIME24 problems, $N{=}16$ samples per condition.}
\label{tab:qualitative_patterns}
\footnotesize
\renewcommand{\arraystretch}{0.95}
\scalebox{0.8594800}{%
\begin{tabular}{lcl}
\toprule
\textbf{Pattern} & \textbf{Frequency} & \textbf{Problems} \\
\midrule
RSP-only correct (Baseline=0/16, RSP$\ge$1/16) & 5/30 (16.7\%) & \#6, \#8, \#14, \#22, \#25 \\
Baseline-only correct (Baseline$\ge$1/16, RSP=0/16) & 0/30 \;\;(0\%) & --- \\
Both correct                          & 6/30 (20.0\%)  & --- \\
Both wrong                            & 19/30 (63.3\%) & --- \\
\bottomrule
\end{tabular}}
\end{table}

Within the $N{=}16$ budget, RSP produces strict gains on five problems and never causes a strict regression. The two case studies below illustrate two faces of the same phenomenon: divergent solution framing (\#6, \S\ref{app:qualitative_case1}) and divergent structural assumption (\#22, \S\ref{app:qualitative_case2}). In both cases, RSP and Baseline diverge already in the opening sentences of the trajectory, and the divergence determines whether the trajectory reaches the correct answer.

\subsection{Case Study 1: AIME24 Problem 6 (Constrained Optimization)}\label{app:qualitative_case1}

\paragraph{Problem.}
\begin{tcolorbox}[colback=gray!5, colframe=black!60, arc=3pt, boxrule=0.5pt, fontupper=\footnotesize, breakable]
Let $\mathcal{B}$ be the set of rectangular boxes with surface area $54$ and volume $23$. Let $r$ be the radius of the smallest sphere that can contain each of the rectangular boxes that are elements of $\mathcal{B}$. The value of $r^2$ can be written as $\frac{p}{q}$, where $p$ and $q$ are relatively prime positive integers. Find $p+q$.
\end{tcolorbox}
Ground truth: $\boxed{721}$, with the optimal box at $a = b = 2$, $c = 23/4$ giving $r^2 = 657/64$.

\paragraph{Aggregate result.} Baseline produces 0/16 correct and RSP produces 1/16 correct. Baseline's 16 wrong final answers span 10 distinct values, with the modal answer $31$ appearing in 4/16 samples, and the correct value $721$ does not appear. RSP's 16 samples produce 11 distinct values and include $721$ exactly once.

\paragraph{Failure mode (Baseline).} Baseline's representative trajectory (final answer 49) opens with an \emph{algebraic identity} frame and never solves for the actual box dimensions.

\begin{tcolorbox}[colback=gray!5, colframe=black!60, arc=3pt, boxrule=0.5pt, fontupper=\footnotesize, breakable, title={Baseline trajectory (final answer 49) --- excerpts}]
\textbf{Opening:} To solve the problem, we start by defining the dimensions of a rectangular box as $a$, $b$, and $c$. The surface area and volume of the box are given by $2(ab + bc + ca) = 54$ and $abc = 23$. We need to find the radius $r$ of the smallest sphere that can contain the box, given by $r = \tfrac{1}{2}\sqrt{a^2 + b^2 + c^2}$. First, we need to find the expression for $a^2 + b^2 + c^2$ in terms of the given equations.

\textbf{Identity manipulation:} Squaring the surface area equation yields $(ab + bc + ca)^2 = 2916$, which expands to $a^2b^2 + b^2c^2 + c^2a^2 + 2abc(a + b + c) = 2916$. Substituting $2abc = 46$ and $(a+b+c)^2 = x + 54$ where $x := a^2 + b^2 + c^2$, then attempting to express $a^4 + b^4 + c^4$ in terms of $x$, the manipulation collapses to the tautology $x^2 = x^2 - 2y + 2y$.

\textbf{Termination by guess:} ``This equation is always true, so we need another way to find $x$. \dots After solving, we find $x = 45$.'' Thus $r^2 = 45/4$, and $p + q = 45 + 4 = \boxed{49}$.
\end{tcolorbox}

\textbf{Diagnosis.} The trajectory frames the task as \emph{find the closed form of $a^2+b^2+c^2$ from the given constraints} and never attempts to solve for the actual dimensions. Two consecutive identity manipulations collapse to a tautology that does not yield $x$, after which the trajectory commits to $x = 45$ without justification (``After solving, we find''). This pattern --- identity manipulation, tautology, terminal guess --- recurs across Baseline's 16 samples. The constraint set under-determines $a^2 + b^2 + c^2$ in closed form, and the algebraic-identity framing has no path to the correct value.

\paragraph{Success mode (RSP).} The correct RSP trajectory (final answer 721) opens with an \emph{optimization} frame and a symmetry assumption that reduces the system to a single cubic.

\begin{tcolorbox}[colback=gray!5, colframe=black!60, arc=3pt, boxrule=0.5pt, fontupper=\footnotesize, breakable, title={RSP trajectory (final answer 721) --- excerpts}]
\textbf{Opening:} To solve the problem, we need to find the dimensions of the rectangular box that \emph{minimizes} the radius of the smallest sphere that can contain it, given the surface area and volume constraints. The radius is $r = \tfrac{1}{2}\sqrt{a^2 + b^2 + c^2}$. To minimize $r$, we minimize $a^2 + b^2 + c^2$.

\textbf{Symmetry reduction:} A more straightforward approach is to use symmetry: $a^2 + b^2 + c^2$ is minimized when $a$, $b$, $c$ are as close to each other as possible. Assume $a = b$ and solve for $c$: $2a^2 + 2ac = 27$ and $a^2 c = 23$. From $a^2 c = 23$, $c = 23/a^2$. Substituting: $2a^2 + 46/a = 27$, multiplying through by $a$ gives $2a^3 - 27a + 46 = 0$.

\textbf{Rational root and termination:} By the Rational Root Theorem, $a = 2$ is a root. Then $c = 23/4$, $a^2 + b^2 + c^2 = 4 + 4 + 529/16 = 657/16$, $r^2 = 657/64$, and $p + q = \boxed{721}$.
\end{tcolorbox}

\textbf{Diagnosis.} Three properties of the opening drive the success: \textbf{(i)} the task is framed as \emph{minimizing} $r$ (and thus $a^2+b^2+c^2$) rather than as a closed-form expression hunt, \textbf{(ii)} the path is committed to solving for $a, b, c$ instead of for $a^2+b^2+c^2$ alone, and \textbf{(iii)} a symmetry assumption $a = b$ reduces the system to one cubic in one unknown. The cubic $2a^3 - 27a + 46 = 0$ has $a = 2$ as a rational root, leading to the explicit configuration $(2, 2, 23/4)$ and the correct $r^2 = 657/64$.

\paragraph{Interpretation.} Baseline's failure is rooted in the \emph{opening sentence}: framing the task as ``find the expression for $a^2+b^2+c^2$ from the given equations'' commits the trajectory to a path that does not have a closed-form solution under the given constraints. RSP's perturbation produces, in 1 of 16 attempts, a trajectory whose opening frame is \emph{constrained optimization} with a symmetry-based variable reduction. The two trajectories share the same model weights, the same problem constraints, and the same sampling temperature. The input-side RSP perturbation is the only intervention, and its visible effect is the different framing the model commits to in the first few sentences.

\subsection{Case Study 2: AIME24 Problem 22 (Constrained List Construction)}\label{app:qualitative_case2}

\paragraph{Problem.}
\begin{tcolorbox}[colback=gray!5, colframe=black!60, arc=3pt, boxrule=0.5pt, fontupper=\footnotesize, breakable]
A list of positive integers has the following properties:
\begin{itemize}\setlength{\itemsep}{0pt}
  \item The sum of the items in the list is $30$.
  \item The unique mode of the list is $9$.
  \item The median of the list is a positive integer that does not appear in the list itself.
\end{itemize}
Find the sum of the squares of all the items in the list.
\end{tcolorbox}
Ground truth: $\boxed{236}$, corresponding to $\{5, 7, 9, 9\}$ with median $\frac{7+9}{2} = 8$.

\paragraph{Aggregate result.} Baseline produces 0/16 correct and RSP produces 3/16 correct. Baseline's 16 samples produce 16 distinct final answers (no two agree), and the correct value 236 does not appear. RSP's 16 samples concentrate slightly more, with 236 (3 times) and 248 (3 times) tied as modal answers and the remaining mass spread across 10 distinct values.

\paragraph{Structural-assumption split.} The constraint ``the median is a positive integer that does not appear in the list'' \emph{requires} an even-length list: for an odd-length list the median is by definition an element of the list. Across Baseline's 16 samples, 6 commit to an odd-length list (typically 5 elements) at the opening and only 3 consider an even-length list. Across RSP's 16 samples, 8 commit to odd and 6 consider even. The shift in the even-length consideration rate (3/16 $\to$ 6/16) is the structural correlate of RSP's accuracy gain.

\paragraph{Failure mode (Baseline).} Baseline's representative trajectory (final answer 224) commits to an odd-length list \emph{at the opening} and never reconsiders.

\begin{tcolorbox}[colback=gray!5, colframe=black!60, arc=3pt, boxrule=0.5pt, fontupper=\footnotesize, breakable, title={Baseline trajectory (final answer 224) --- excerpts}]
\textbf{Opening commitment:} ``Let's start by assuming the list has an odd number of elements. If the list has $2k+1$ elements, the median is the $(k+1)$-th element. Let's try with the smallest possible odd number\dots''

\textbf{Construction:} 5 elements, $a_1 + a_2 + 9 + a_4 + 9 = 30$ so $a_1 + a_2 + a_4 = 12$. Try $a_1 = 2$, $a_2 = 3$, $a_4 = 7$: list $\{2, 3, 7, 9, 9\}$, median $7$, ``which does not appear in the list, so this is a valid candidate.''

\textbf{Termination:} Sum of squares $= 4 + 9 + 49 + 81 + 81 = \boxed{224}$.
\end{tcolorbox}

\textbf{Diagnosis.} Two errors compound: (i) the trajectory \emph{defaults to odd length} at the opening without justification, and the constraint ``median not in the list'' forces even length, which the trajectory never derives, and (ii) the median value $7$ is verified as a positive integer but \emph{not} verified as absent from the list ($7$ is in fact an element). \rspadd{The same structural assumption recurs in several Baseline samples. The elements vary, but the odd-length commitment and incomplete constraint check remain.}

\paragraph{Success mode (RSP).} The correct RSP trajectory (final answer 236) considers $n = 4$ from the outset and treats the median formula by parity.

\begin{tcolorbox}[colback=gray!5, colframe=black!60, arc=3pt, boxrule=0.5pt, fontupper=\footnotesize, breakable, title={RSP trajectory (final answer 236) --- excerpts}]
\textbf{Opening (Step 2: median by parity).} ``If the list has an odd number of elements, the median is the middle element. If the list has an even number of elements, the median is the average of the two middle elements. Since the median is a positive integer that does not appear in the list, it must be an integer and not equal to any of the other numbers in the list.''

\textbf{Step 3: starting at $n = 4$.} ``$n \ge 4$ (at least two 9's). Case 1: $n = 4$. Sum: $9 + 9 + x + y = 30$, so $x + y = 12$.''

\textbf{Construction:} Try $x = 5$, $y = 7$: list $\{5, 7, 9, 9\}$, median $\tfrac{7+9}{2} = 8$ (integer, not in list). All conditions satisfied, so the sum of squares $= 25 + 49 + 81 + 81 = \boxed{236}$.
\end{tcolorbox}

\textbf{Diagnosis.} Three properties distinguish this trajectory: \textbf{(i)} \emph{even length is considered first} ($n = 4$ enumerated as the initial case), \textbf{(ii)} the median is computed as an arithmetic mean for even length, producing $8 \notin \{5, 7, 9, 9\}$, and \textbf{(iii)} the constraint check is performed against the candidate \emph{value} ($8$), not against the structural position. The other two correct RSP samples reach 236 through partially incorrect intermediate reasoning, suggesting that some RSP samples that begin with the same odd-length assumption as Baseline can still arrive at the correct numerical answer.

\paragraph{Interpretation.} Baseline's failure is rooted in the \emph{first structural assumption} after the conditions are listed: the trajectory commits to an odd-length list within its first few hundred tokens and never reconsiders. RSP's perturbation increases the fraction of samples that consider $n = 4$ from 3/16 to 6/16, and one of these reaches the correct constraint check. The split between odd-length and even-length openings is the empirical fingerprint of the template-selection shift. On a small fraction of samples, the alternative template happens to be the structurally consistent one.

\subsection{Summary of Case Studies}\label{app:qualitative_summary}

Both case studies exhibit the same shape. Baseline's 16 samples cluster around a single failure template (algebraic-identity framing for \#6 and odd-length-list commitment for \#22) and never break out of it within the sample budget. RSP does not change the model's underlying knowledge: most RSP samples fall into the same templates as Baseline. What RSP changes is \emph{which template the model commits to in the opening sentences}. On a small fraction of trajectories, the alternative template happens to admit a solvable subproblem (the optimization framing for \#6, the even-length list for \#22), and the trajectory reaches the correct answer. The accuracy lift (0/16 $\to$ 1/16 on \#6, 0/16 $\to$ 3/16 on \#22) is the outcome-level signature of this opening-sentence divergence, \rspadd{consistent with the input-side branch-opening mechanism of \S\ref{sec:improve}, where RSP can alter early decoding decisions before the trajectory commits to a structural template.}

\section{Prompt-Law Properties Used by the Theory}\label{app:why_random}

This appendix isolates the only properties of the RSP law used in the main text: centeredness, direction-agnostic covariance under a variance budget, and positive probability on every nonempty open set. We avoid stronger semantic claims because Section~\ref{sec:multipath} needs only these geometric and measure-theoretic facts.

\subsection{Centered Perturbations Do Not Impose Systematic First-Order Drift}

Let $z^{\text{Base}}_a$ denote the no-RSP output logit at vocab token $a$ and $z^{\text{RSP}}_a(\bar h)$ the corresponding logit when the centered RSP perturbation takes value $\bar h$. The (true) vocab-logit gap under RSP is $\Delta_{ab}(\bar h) := z^{\text{RSP}}_a(\bar h) - z^{\text{RSP}}_b(\bar h)$, distinct from the no-RSP baseline gap $\Delta^{\text{Base}}_{ab} := z^{\text{Base}}_a - z^{\text{Base}}_b$. Note that $\bar h = 0$ corresponds to the centered RSP at its entrywise mean rather than to the no-RSP forward pass, so in general $\Delta_{ab}(0) \ne \Delta^{\text{Base}}_{ab}$. The transformer's logit map is non-linear in $\bar h$, but a first-order Taylor expansion around $\bar h = 0$ yields the local surrogate
\[
\Delta_{ab}^{\mathrm{lin}}(\bar{h}) \;:=\; \Delta_{ab} + b_{ab}^\top \bar{h},
\]
with $\Delta_{ab} := \Delta_{ab}(0)$ and $b_{ab} := \nabla_{\bar h}(z^{\text{RSP}}_a - z^{\text{RSP}}_b)|_{\bar h = 0}$. The proposition below is stated for the surrogate $\Delta_{ab}^{\mathrm{lin}}$. For the true non-linear gap $\Delta_{ab}(\bar h)$, the first-order term has zero mean under centered perturbations, but the Taylor remainder $r_{ab}(\bar h) = O(\|\bar h\|^2)$ may contribute a second-order mean shift. We make no claim about its sign or magnitude.

\begin{proposition}[No systematic first-order drift in the surrogate]\label{prop:centered_no_drift}
If $\mathbb{E}[\bar{h}] = 0$, then
\[
\mathbb{E}\big[\Delta_{ab}^{\mathrm{lin}}(\bar{h})\big] = \Delta_{ab}
\]
for every token pair $(a,b)$. If a deterministic centered prompt $\bar h_0$ is reused across runs, then
\[
\Delta_{ab}^{\mathrm{lin}}(\bar h_0) - \Delta_{ab} = b_{ab}^\top \bar h_0
\]
is a fixed directional bias.
\end{proposition}

\begin{proof}
The centered case is immediate from linearity of expectation:
\[
\mathbb{E}\big[\Delta_{ab}^{\mathrm{lin}}(\bar{h})\big]
=
\Delta_{ab} + b_{ab}^\top \mathbb{E}[\bar{h}]
=
\Delta_{ab}.
\]
The deterministic case follows by substitution.
\end{proof}

This proposition rules out a run-averaged first-order bias. It does not say that individual prompt draws leave the logits unchanged.

\subsection{Proof of Proposition~\ref{thm:minimax_directional_coverage}}

For a unit vector $u$, the first-order perturbation energy available along $u$ is $\mathrm{Var}_{h\sim D}(u^\top h)=u^\top \Sigma_D u$. The worst-case directional energy is therefore the minimum Rayleigh quotient of $\Sigma_D$.

\begin{proof}
For any symmetric covariance matrix $\Sigma_D$,
\[
\min_{\|u\|=1} u^\top \Sigma_D u = \lambda_{\min}(\Sigma_D).
\]
Let the eigenvalues of $\Sigma_D$ be $\lambda_1,\dots,\lambda_d$. Then
\[
\lambda_{\min}(\Sigma_D)
\le
\frac{1}{d}\sum_{m=1}^{d}\lambda_m
=
\frac{\mathrm{tr}(\Sigma_D)}{d}
\le
\frac{\rho^2}{d}.
\]
Equality holds if and only if all eigenvalues are equal, i.e., $\Sigma_D = (\rho^2/d)\mathbf{I}_d$.
\end{proof}

The proof uses only covariance. We choose a Gaussian law for RSP because it adds the open-set reachability property proved next.

\subsection{Open-Set Reachability of Isotropic Gaussian Prompts}

\begin{proposition}[Open-set reachability]\label{prop:gaussian_open_support}
Let $\bar h \sim \mathcal{N}(0,\sigma^2\mathbf{I}_d)$ with $\sigma>0$. Then every nonempty open set $U \subseteq \mathbb{R}^d$ satisfies
\[
\Pr(\bar h \in U) > 0.
\]
\end{proposition}

\begin{proof}
The Gaussian density
\[
(2\pi \sigma^2)^{-d/2}\exp\!\left(-\frac{\|\bar h\|^2}{2\sigma^2}\right)
\]
is strictly positive at every point of $\mathbb{R}^d$. Every nonempty open set contains a ball of positive Lebesgue measure, and integrating a strictly positive density over that ball gives positive probability.
\end{proof}

This is the only support property used in the informal branching argument of \S\ref{sec:improve} and in Appendix~\ref{app:topk_entry}.

\section{Proofs for the Transformer-Level Mechanism}\label{app:proofs}

This appendix follows the same order as the main theory section: exploration, annealing, and performance gain. The prompt-law facts used before these proofs are collected in Appendix~\ref{app:why_random}. The attention results below, Eq.~\eqref{eq:decomp} and Theorem~\ref{thm:decay}, analyze one head at a time. They are a local diagnosis of how attention routes the RSP contribution, not a complete causal account of how the perturbation propagates through the full Transformer.

\subsection{Exploration: attention decomposition and perturbation size}\label{app:rsp_noise}

\begin{proposition}[Attention decomposition]\label{prop:attention_decomposition}
For one attention head at query position $i$ in layer $\ell$, let $\alpha_{ij}^{(\ell)}$ be the attention weights over $n\ge1$ unmasked real tokens and $L\ge1$ random tokens, with finite attention logits. Define
\[
w_{r,i}^{(\ell)} := \sum_{j=1}^{L} \alpha_{i,n+j}^{(\ell)}
\]
and, for $j \in [n]$,
\[
\tilde{\alpha}_{ij}^{(\ell)} := \frac{\alpha_{ij}^{(\ell)}}{1-w_{r,i}^{(\ell)}}.
\]
Then
\[
\mathbf{o}_i^{(\ell)}
=
(1-w_{r,i}^{(\ell)})\,\tilde{\mathbf{o}}_i^{(\ell)}
+
\boldsymbol{\eta}_i^{(\ell)},
\]
where
\[
\tilde{\mathbf{o}}_i^{(\ell)} := \sum_{j=1}^{n} \tilde{\alpha}_{ij}^{(\ell)} \mathbf{v}_j^{(\ell)}
\qquad\text{and}\qquad
\boldsymbol{\eta}_i^{(\ell)} := \sum_{j=1}^{L} \alpha_{i,n+j}^{(\ell)} \mathbf{v}_{r_j}^{(\ell)}.
\]
\end{proposition}

\begin{proof}
The head output is
\[
\mathbf{o}_i^{(\ell)}
=
\sum_{j=1}^{n} \alpha_{ij}^{(\ell)} \mathbf{v}_j^{(\ell)}
+
\sum_{j=1}^{L} \alpha_{i,n+j}^{(\ell)} \mathbf{v}_{r_j}^{(\ell)}.
\]
By the softmax positivity and the existence of at least one unmasked real token, $\sum_{j=1}^{n} \alpha_{ij}^{(\ell)} = 1-w_{r,i}^{(\ell)} > 0$, so
\[
\sum_{j=1}^{n} \alpha_{ij}^{(\ell)} \mathbf{v}_j^{(\ell)}
=
(1-w_{r,i}^{(\ell)}) \sum_{j=1}^{n} \frac{\alpha_{ij}^{(\ell)}}{1-w_{r,i}^{(\ell)}} \mathbf{v}_j^{(\ell)}
=
(1-w_{r,i}^{(\ell)})\,\tilde{\mathbf{o}}_i^{(\ell)}.
\]
Substituting this identity gives the decomposition.
\end{proof}

\begin{corollary}[Magnitude scales with random-token attention]\label{cor:perturbation_magnitude}
For every realization of the random values,
\[
\|\boldsymbol{\eta}_i^{(\ell)}\|
\le
w_{r,i}^{(\ell)} \max_{j \in [L]} \|\mathbf{v}_{r_j}^{(\ell)}\|.
\]
\end{corollary}

\begin{proof}
By the triangle inequality,
\[
\|\boldsymbol{\eta}_i^{(\ell)}\|
=
\left\|
\sum_{j=1}^{L} \alpha_{i,n+j}^{(\ell)} \mathbf{v}_{r_j}^{(\ell)}
\right\|
\le
\sum_{j=1}^{L} \alpha_{i,n+j}^{(\ell)} \|\mathbf{v}_{r_j}^{(\ell)}\|
\le
w_{r,i}^{(\ell)} \max_{j \in [L]} \|\mathbf{v}_{r_j}^{(\ell)}\|.
\]
\end{proof}

\rspadd{Corollary~\ref{cor:perturbation_magnitude} bounds the norm of the direct RSP contribution. This bound is distinct from the attention-mass identity in Theorem~\ref{thm:decay}.} No independence assumption between attention weights and random keys is required. \rspadd{For bounded RSP value norms, the upper bound on the random contribution decreases with $w_{r,i}^{(\ell)}$.}

\subsection{Corollaries of Theorem~\ref{thm:decay} on KV cache growth}

{\ifrspchanges\color{rspadded}\fi
Theorem~\ref{thm:decay} expresses the attention mass as $w_{r,i}^{(\ell)}=L/(L+n e^{g_i^{(\ell)}})$, with a gap $g_i^{(\ell)}$ that may vary throughout generation. Thus, the decrease of RSP attention depends on the product $n e^{g_i^{(\ell)}}$, rather than on cache growth alone. The following corollaries state the conditions for attention decay and, separately, for the direct RSP contribution to vanish.

\begin{corollary}[Sequence-wise annealing]\label{cor:sequence_annealing}
For fixed $L$, RSP attention decreases whenever $n e^{g_i^{(\ell)}}$ increases. If $n e^{g_i^{(\ell)}}\to\infty$, then $w_{r,i}^{(\ell)}\to0$.
\end{corollary}
\begin{proof}
Both statements follow directly from $w_{r,i}^{(\ell)}=L/(L+n e^{g_i^{(\ell)}})$. The condition permits a varying gap throughout generation.
\end{proof}

\begin{corollary}[Vanishing random contribution under annealing]\label{cor:vanishing_perturbation}
For fixed $L$, if $n e^{g_i^{(\ell)}}\to\infty$ and the RSP value norms remain uniformly bounded, then $\|\boldsymbol{\eta}_i^{(\ell)}\|\to0$.
\end{corollary}
\begin{proof}
Corollary~\ref{cor:sequence_annealing} gives $w_{r,i}^{(\ell)}\to0$. Corollary~\ref{cor:perturbation_magnitude} then gives the result under the bounded-value assumption, which holds for fixed, finite RSP values in the causal KV cache.
\end{proof}

\subsection{Attention-logit gap along generation}\label{app:gap_dynamics}
Under the same setting as Figure~\ref{fig:attention_mass} (Qwen2.5-Math-7B, \textit{suffix}, MATH-500), Table~\ref{tab:gap_dynamics} reports the token-count-normalized log-sum-exp gap averaged over heads and samples. We interpret changes within each layer band along generation. The gap increases most clearly in the shallowest and deepest bands, while the middle bands show smaller and non-monotonic changes. These measurements complement the attention heatmaps. Equation~\eqref{eq:rsp_gap_mass} holds for each head and query, so substituting the mean gap into it does not in general recover the mean attention mass.

\begin{table}[ht]
\centering
\caption{Mean token-count-normalized log-sum-exp gap by layer depth and generation position.}
\label{tab:gap_dynamics}
\footnotesize
\scalebox{0.8594800}{%
\begin{tabular}{lccccc}
\toprule
Layer depth (\%) & \multicolumn{5}{c}{Generation position (\%)} \\
\cmidrule(lr){2-6}
 & 0--20 & 20--40 & 40--60 & 60--80 & 80--100 \\
\midrule
0--20   & 2.56 & 3.76 & 4.66 & 5.69 & 6.74 \\
20--40  & 2.45 & 3.00 & 3.16 & 3.30 & 3.26 \\
40--60  & 3.47 & 3.80 & 3.63 & 3.53 & 3.17 \\
60--80  & 3.70 & 4.00 & 3.93 & 3.89 & 3.59 \\
80--100 & 4.34 & 6.07 & 7.65 & 8.98 & 9.75 \\
\bottomrule
\end{tabular}}
\end{table}
}

\subsection{Performance gain: from local admission to Pass@$N$}

From here we work at the vocab-logit and task levels, dropping the per-layer index $\ell$. \rspadd{Section~\ref{sec:improve} uses a local first-order surrogate to analyze branch opening. Section~\ref{sec:converge} connects independent rollouts to Pass@$N$ through assumptions (M1) and (M2). The branch-opening argument below uses the following support condition on the centered prompt law $D$:}
\[
\textnormal{(S1)}\qquad D(U) > 0 \quad \text{for every nonempty open set } U \subseteq \mathbb{R}^d.
\]
Proposition~\ref{prop:gaussian_open_support} shows that the isotropic Gaussian law of §\ref{sec:multipath} satisfies (S1).

\subsubsection{Autoregressive Formulation}

In autoregressive decoding, the joint distribution over a sequence $y_{1:T}$ factorizes as
\[
\Pr(y_{1:T} \mid x) = \prod_{t=1}^T p(y_t \mid x, y_{<t}),
\]
where each step is governed by a prefix-conditioned distribution. Any perturbation introduced by RSP affects generation through a sequence of local changes to $p(y_t \mid x, y_{<t})$, rather than a single global distribution.

\subsubsection{Local Linearization of Logits under RSP}

Throughout this subsection, $\bar h\in\mathbb{R}^d$ denotes \emph{one} centered prompt vector from the RSP definition (Eq.~\eqref{eq:rsp_def} in §\ref{sec:rsp_definition}), treated as a per-token surrogate. The actual implementation uses a length-$L$ prompt $\mathbf{H}=(h_1,\dots,h_L)$ with $L$ independent draws, and the arguments below isolate how one local perturbation can shift the logits. We model the perturbed logits as a function $z_a(\bar{h})$ at vocab token $a$, assuming local smoothness in a neighborhood of $\bar{h}=0$:
\[
z_a(\bar{h}) = z_a(0) + \nabla_{\bar h} z_a(0)^\top \bar{h} + r_a(\bar{h}),
\]
where $r_a(\bar{h}) = O(\|\bar{h}\|^2)$. Defining $c_a := \nabla_{\bar h} z_a(0)$ and $z_a := z_a(0)$, we obtain the local affine surrogate
\[
z_a^{\mathrm{lin}}(\bar{h}) := z_a + c_a^\top \bar{h}.
\]
This approximation is used as a first-order surrogate for branch opening, not as an exact global model of the transformer logits.

\subsubsection{Top-$K$ entry region (Remark)}\label{app:topk_entry}

This complements the informal branching argument in \S\ref{sec:improve}. Define the entry region
\[
\mathcal{G}_{a,K}
\;:=\;
\bigl\{\bar h \in \mathbb{R}^d : \#\{b\neq a : z_b^{\mathrm{lin}}(\bar h) > z_a^{\mathrm{lin}}(\bar h)\} \le K-1\bigr\}.
\]
If $\mathcal{G}_{a,K}$ contains a nonempty open set $U$, then by (S1) $D(U)>0$, and $U\subseteq\mathcal{G}_{a,K}$ gives $\Pr_{\bar h}(\bar h\in\mathcal{G}_{a,K})\ge D(U)>0$. A useful sufficient condition is the existence of some $\bar h_0$ at which token $a$ strictly outranks all but at most $K-1$ other tokens in the affine surrogate. By continuity, a neighborhood of $\bar h_0$ is then contained in $\mathcal{G}_{a,K}$. This is a rank-changing event unreachable by the monotone temperature transform, but it does not by itself imply (M1).

\subsubsection{Pass@$N$ ensemble bound (Remark)}\label{app:passN_ensemble}

The bound $\Pr(\mathrm{Pass@}N \text{ on } x) \ge 1-(1-p_{\min}(x))^N$ in the main text is a standard independent-Bernoulli union argument. Let $A_n$ be the event that the $n$-th run reaches a correct trajectory on $x$. Under (M1) $\Pr(A_n)\ge p_{\min}(x)$, and under (M2)
\[
\Pr\!\Big(\bigcap_{n=1}^N A_n^c\Big)
\le (1-p_{\min}(x))^N,
\qquad
\Pr\!\Big(\bigcup_{n=1}^N A_n\Big)
\ge 1-(1-p_{\min}(x))^N
\;\ge\; 1-e^{-N p_{\min}(x)}.
\]
This is a generic fact, not specific to RSP. RSP's role is to make branch opening possible and to support (M2) through independent resampling, while whether a branch satisfies the task-side correctness condition (M1) remains task-dependent.

\section{Attention Mass Analysis}\label{app:attention_mass}

This appendix formalizes the per-token attention mass reported in Figure~\ref{fig:attention_mass} and the exclusion criteria applied to the reasoning span and the layer axis. For each of the 500 MATH-500 problems, a fresh RSP $\mathbf{H} \in \mathbb{R}^{L \times d}$ with $L = 10$ is sampled independently, so the reported heatmaps average over both problem variability and RSP-vector variability.

\subsection{Per-Token Attention Mass}

For each problem, we measure attention on the concatenated sequence $[\text{prompt}; \mathbf{H}; \text{generation}]$, where the generation is the trajectory produced by the same model under matching RSP injection. Let $\alpha^{(\ell, m)}_{i,j}$ denote the post-softmax attention weight at layer $\ell$, head index $m$ (using $m$ to avoid collision with the vocab token $a$ of \S\ref{sec:improve} and the RSP matrix $\mathbf{H}$), from query position $i$ to key position $j$, satisfying $\sum_j \alpha^{(\ell, m)}_{i,j} = 1$ under the causal mask. Let $Q$ and $R$ denote the index sets of question and RSP tokens with sizes $|Q|$ and $|R| = L$, and let $N_{\text{head}}$ be the number of heads. The per-token attention mass that query $i$ directs to each group at layer $\ell$ is
\begin{equation}
\mathrm{PTM}_{Q}[\ell, i] = \frac{1}{|Q|\,N_{\text{head}}} \sum_{m=1}^{N_{\text{head}}} \sum_{j \in Q} \alpha^{(\ell, m)}_{i,j},
\qquad
\mathrm{PTM}_{R}[\ell, i] = \frac{1}{|R|\,N_{\text{head}}} \sum_{m=1}^{N_{\text{head}}} \sum_{j \in R} \alpha^{(\ell, m)}_{i,j}.
\label{eq:per_token_mass}
\end{equation}
Normalization by $|Q|$ and $|R|$ controls for group size, so both quantities express how much attention a single question (resp.\ RSP) token receives on average under the same softmax denominator. Question tokens thereby serve as a size-matched baseline that absorbs sequence-length effects common to both groups.

\subsection{Heatmap Aggregation}

We partition the layer indices and the reasoning position indices each into five equal-size bins $\{B_L^{(b)}\}_{b=1}^{5}$ and $\{B_P^{(b)}\}_{b=1}^{5}$. For sample $s$ and target group $\bullet \in \{Q, R\}$, the value in cell $(b_L, b_P)$ is
\begin{equation}
M^{(s)}_{\bullet}[b_L, b_P] = \frac{1}{|B_L^{(b_L)}| \cdot |B_P^{(b_P)}|} \sum_{\ell \in B_L^{(b_L)}} \sum_{i \in B_P^{(b_P)}} \mathrm{PTM}_{\bullet}[\ell, i].
\label{eq:binned_mass}
\end{equation}
Figure~\ref{fig:attention_mass} reports the sample average of Eq.~\eqref{eq:binned_mass} over the 500 valid samples.

\subsection{Excluding the \texttt{\textbackslash boxed\{\ldots\}} Span}

The reasoning position range terminates strictly before the final \texttt{\textbackslash boxed\{\ldots\}} span. Chain-of-thought outputs decompose into a \textit{text} (content) component and a \textit{pattern} (template) component that contribute independently to downstream performance \citep{madaan2022text}. The \texttt{\textbackslash boxed\{\ldots\}} wrapper is a canonical pattern: a short, stereotyped span whose role is to complete the answer format rather than to extend reasoning. Including it would therefore inject a template-driven attention signature unrelated to reasoning dynamics.

\subsection{Excluding the Final Two Layers}

The layer axis further excludes the last two transformer layers. This choice is motivated by a standard late-layer effect rather than by any model-specific tuning decision.

\paragraph{Output-projection specialization.}
Late-layer hidden states lie in an approximately linear relationship with vocabulary logits and therefore function as a progressive linear readout rather than as representation-transforming blocks \citep{belrose2023eliciting}. Consistently, late-layer feed-forward modules have been characterized as output-distribution shaping mechanisms \citep{geva2021transformer}. Attention in these layers therefore reflects output formation rather than reasoning computation.

Excluding the final two layers keeps the heatmap focused on reasoning-phase dynamics instead of readout-phase dynamics. This exclusion is not load-bearing for the main claim: the sequence-direction attenuation emphasized in the main text is also visible without it, and \rspadd{Theorem~\ref{thm:decay} gives an exact per-head identity, with generation and layer trends assessed empirically}.

\subsection{Additional \textit{Suffix} Heatmaps}

For the late-layer reasons discussed above, the pattern in Figure~\ref{fig:app_attention_additional} does not generalize uniformly across every cell. Nevertheless, the overall trend is consistent across both models: question-token attention varies broadly across layers, whereas RSP attention shows sequence-wise decay \rspadd{interpreted through the exact identity in Theorem~\ref{thm:decay}, alongside an empirical layer-wise attenuation}.

\begin{figure}[H]
  \centering
  \begin{subfigure}[b]{\textwidth}
    \centering
    \includegraphics[width=\textwidth]{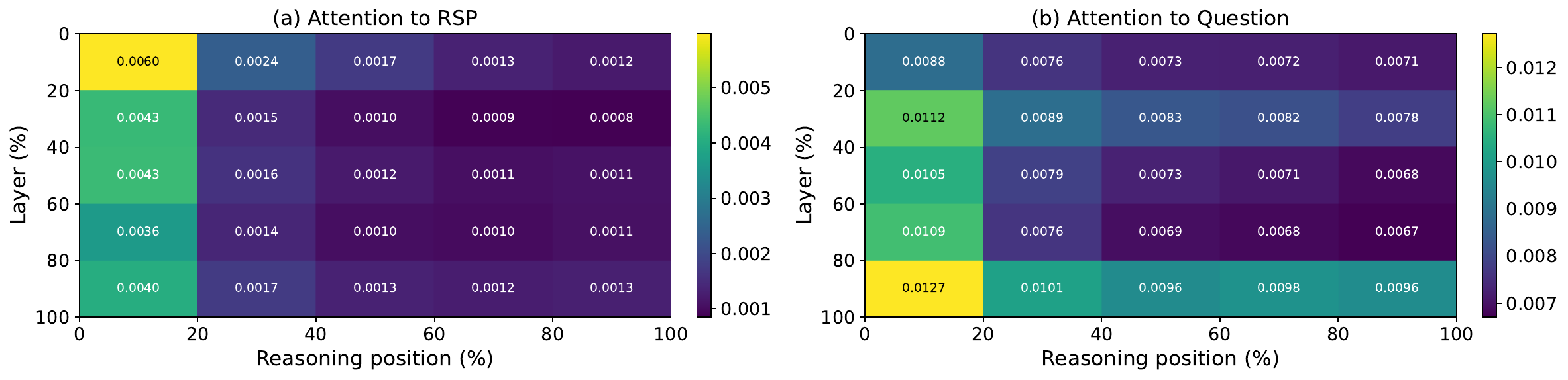}
    \caption{Qwen2.5-Math-1.5B-Instruct, \textit{suffix}}
  \end{subfigure}
  \\[0.5em]
  \begin{subfigure}[b]{\textwidth}
    \centering
    \includegraphics[width=\textwidth]{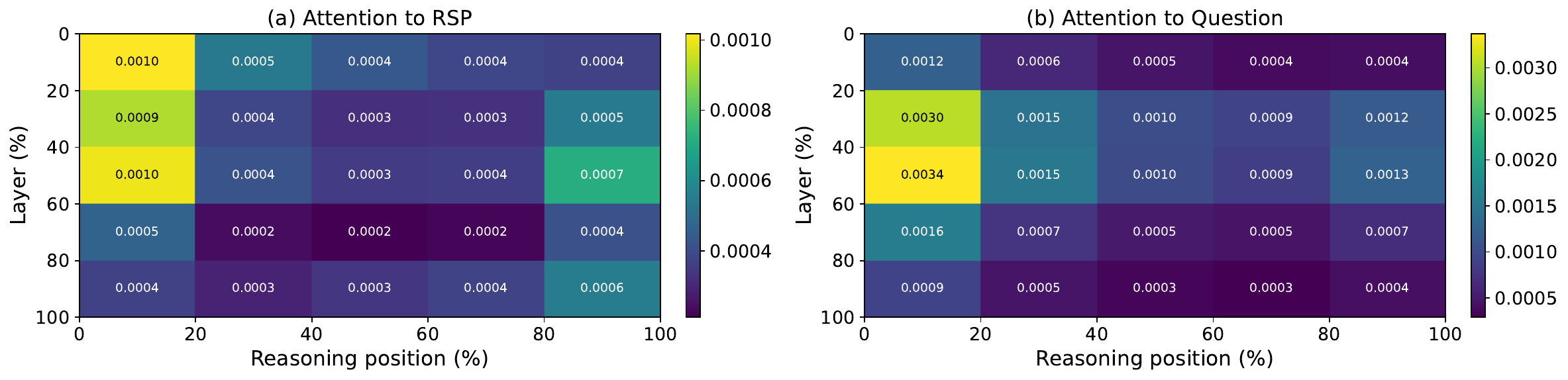}
    \caption{LLaMA-3.1-8B-Instruct, \textit{suffix}}
  \end{subfigure}
  \caption{Per-token RSP attention mass under \textit{suffix} injection for the remaining two models (500 MATH-500 samples each). Axes and preprocessing match Figure~\ref{fig:attention_mass}: x-axis is reasoning position binned into five quantiles, y-axis is layer depth binned into five quantiles (top = shallow), color encodes the 500-sample mean per-token attention assigned to RSP tokens, and tokens inside \texttt{\textbackslash boxed\{\}} spans and the last two layers are excluded.}
  \label{fig:app_attention_additional}
\end{figure}

\input{appendix/dapo}

\section{Broader Impacts}\label{app:broader_impacts}

This work provides empirical and theoretical analysis of how random embedding injection affects the reasoning behavior of large language models. \rspadd{The contributions are primarily foundational, with potential applications to training-data collection for reasoning models.}

\paragraph{Positive impacts.}
\rspadd{RSP-induced trajectory diversity can help collect correct solutions for problems that the initial model solves infrequently. Our SFT experiments suggest that such data collection can improve subsequent model performance without requiring RSP at inference. This could make self-generated training data more useful for reasoning-oriented fine-tuning.}

\paragraph{Potential concerns.}
Methods that expand the effective output support of LLMs may increase the reachability of low-probability tokens, which includes both desirable alternative reasoning paths and potentially unsafe or off-distribution outputs. Practitioners applying RSP in production systems should combine it with existing safety filtering and alignment mechanisms rather than treating it as a standalone diversity source.

%% file: appendix/sft_setup.tex
\subsection{Rejection-Sampling SFT Setup}\label{app:sft_setup}

\rspadd{We use Qwen2.5-Math-7B-Instruct for both data collection and the initial SFT model. The source pool contains 45{,}792 problems from OpenR1-Math-46k~\citep{yan2025learning}\footnote{\url{https://huggingface.co/datasets/Elliott/Openr1-Math-46k-8192}}. The fixed hard subset contains 32{,}737 problems with at most one correct solution among four clean rollouts. Each collection condition uses eight rollouts per hard problem at temperature $1.0$. RSP uses 20 \textit{suffix} vectors, with eight independently drawn groups for the eight rollout slots. The collection implementation samples the vectors from $\mathcal{N}(0,0.0226^2\mathbf I)$ and reuses each group across problems within a run. Collection seeds 42, 43, and 44 control RSP generation and decoding randomness for RSP, and decoding randomness for Baseline. Only correct trajectories are retained.}

\rspadd{For each collected dataset, we fine-tune the same initial model for one epoch using response-only loss. We use a learning rate of $5\times10^{-6}$, cosine decay, a warmup ratio of $0.05$, weight decay $0.01$, an effective batch size of 32, and a maximum sequence length of 4{,}096. The SFT seed is fixed at 42, so the reported variation is across collection seeds rather than independent SFT initialization seeds. SFT uses eight NVIDIA H100 GPUs with 80GB memory each. Logged training runtimes per run are approximately 15 minutes for Baseline-SFT and 55 minutes for RSP-SFT, excluding data collection and evaluation.}

\rspadd{Evaluation uses no RSP and a maximum generation length of 4{,}096 tokens. Minerva Math, GaoKao2023-EN, and OlympiadBench use greedy decoding. AIME24 and AMC23 use 32 samples per problem at temperature $1.0$ and top-$p=0.95$, reporting Avg@32. The reported average is the unweighted mean of these five benchmarks.}

%% file: appendix/content_probes_7b.tex
\section{Content-Probe Replication at 7B}\label{app:content_probes_7b}

We replicate probes (i) and (ii) of \S\ref{sec:content_probes} on Qwen2.5-Math-7B-Instruct, retaining each method's injection budget and position (20 \textit{prefix} tokens for TTSV, 8 \textit{infix} tokens for LTPO, and 4 \textit{infix} tokens for SoftCoT). SoftCoT uses Qwen2.5-Math-1.5B-Instruct as its assistant. Evaluation uses greedy decoding with the \texttt{mathstral} prompt and a chat template. Results are mean accuracy and standard deviation over three seeds: 15--17 for TTSV and 42--44 for LTPO and SoftCoT. Table~\ref{tab:content_probes_7b} reports the conditions corresponding to the main-text table. ``50\% random'' denotes replacement at a randomly selected half of the injected positions.

\begin{table}[!ht]
  \centering
  \caption{\textbf{Replication of the content probes at 7B.} Accuracy (\%, mean $\pm$ std, three seeds) on Qwen2.5-Math-7B-Instruct. Injection length and position remain fixed within each method. The three replacement rows use RSP, and the final two rows disrupt learned structure.}
  \label{tab:content_probes_7b}
  \footnotesize
  \setlength{\tabcolsep}{5pt}
  \renewcommand{\arraystretch}{0.9}
  \scalebox{0.8594800}{%
\begin{tabular*}{\rspPerformanceTableWidth}{@{\extracolsep{\fill}}llccc@{}}
    \toprule
    \textbf{Method} & \textbf{Intervention} & \textbf{MATH-500} & \textbf{GSM8K} & \textbf{AIME24} \\
    \midrule
    \textbf{TTSV} & \textbf{Trained} & $83.60$\,{\tiny\color{black!65}$\pm 0.43$} & $95.02$\,{\tiny\color{black!65}$\pm 0.16$} & $15.56$\,{\tiny\color{black!65}$\pm 1.57$} \\
     & 50\% front & $81.80$\,{\tiny\color{black!65}$\pm 0.16$} & $94.95$\,{\tiny\color{black!65}$\pm 0.04$} & $15.56$\,{\tiny\color{black!65}$\pm 4.16$} \\
     & 50\% random & $82.67$\,{\tiny\color{black!65}$\pm 1.23$} & $95.15$\,{\tiny\color{black!65}$\pm 0.21$} & $14.44$\,{\tiny\color{black!65}$\pm 3.14$} \\
     & 100\% RSP & $82.13$\,{\tiny\color{black!65}$\pm 0.34$} & $94.95$\,{\tiny\color{black!65}$\pm 0.23$} & $18.89$\,{\tiny\color{black!65}$\pm 3.14$} \\
    \cmidrule(l){2-5}
     & Token-shuffle & $83.47$\,{\tiny\color{black!65}$\pm 1.72$} & $94.92$\,{\tiny\color{black!65}$\pm 0.12$} & $13.33$\,{\tiny\color{black!65}$\pm 2.72$} \\
     & Dim-permute & $82.13$\,{\tiny\color{black!65}$\pm 1.98$} & $94.82$\,{\tiny\color{black!65}$\pm 0.26$} & $14.44$\,{\tiny\color{black!65}$\pm 5.67$} \\
    \midrule
    \textbf{LTPO} & \textbf{Trained} & $82.87$\,{\tiny\color{black!65}$\pm 0.34$} & $94.92$\,{\tiny\color{black!65}$\pm 0.16$} & $13.33$\,{\tiny\color{black!65}$\pm 2.72$} \\
     & 50\% front & $81.73$\,{\tiny\color{black!65}$\pm 0.25$} & $94.79$\,{\tiny\color{black!65}$\pm 0.14$} & $15.56$\,{\tiny\color{black!65}$\pm 1.57$} \\
     & 50\% random & $81.87$\,{\tiny\color{black!65}$\pm 0.19$} & $94.87$\,{\tiny\color{black!65}$\pm 0.36$} & $12.22$\,{\tiny\color{black!65}$\pm 1.57$} \\
     & 100\% RSP & $82.87$\,{\tiny\color{black!65}$\pm 0.57$} & $95.00$\,{\tiny\color{black!65}$\pm 0.12$} & $11.11$\,{\tiny\color{black!65}$\pm 1.57$} \\
    \cmidrule(l){2-5}
     & Token-shuffle & $82.67$\,{\tiny\color{black!65}$\pm 0.50$} & $95.10$\,{\tiny\color{black!65}$\pm 0.07$} & $14.44$\,{\tiny\color{black!65}$\pm 3.14$} \\
     & Dim-permute & $83.20$\,{\tiny\color{black!65}$\pm 0.16$} & $94.79$\,{\tiny\color{black!65}$\pm 0.19$} & $13.33$\,{\tiny\color{black!65}$\pm 0.00$} \\
    \midrule
    \textbf{SoftCoT} & \textbf{Trained} & $83.67$\,{\tiny\color{black!65}$\pm 0.34$} & $94.57$\,{\tiny\color{black!65}$\pm 0.25$} & $17.78$\,{\tiny\color{black!65}$\pm 1.57$} \\
     & 50\% front & $83.33$\,{\tiny\color{black!65}$\pm 0.19$} & $94.69$\,{\tiny\color{black!65}$\pm 0.06$} & $15.56$\,{\tiny\color{black!65}$\pm 1.57$} \\
     & 50\% random & $83.07$\,{\tiny\color{black!65}$\pm 0.19$} & $94.72$\,{\tiny\color{black!65}$\pm 0.04$} & $13.33$\,{\tiny\color{black!65}$\pm 0.00$} \\
     & 100\% RSP & $82.53$\,{\tiny\color{black!65}$\pm 0.25$} & $94.59$\,{\tiny\color{black!65}$\pm 0.13$} & $16.67$\,{\tiny\color{black!65}$\pm 2.72$} \\
    \cmidrule(l){2-5}
     & Token-shuffle & $82.93$\,{\tiny\color{black!65}$\pm 0.19$} & $94.74$\,{\tiny\color{black!65}$\pm 0.09$} & $14.44$\,{\tiny\color{black!65}$\pm 1.57$} \\
     & Dim-permute & $83.67$\,{\tiny\color{black!65}$\pm 0.68$} & $94.97$\,{\tiny\color{black!65}$\pm 0.13$} & $15.56$\,{\tiny\color{black!65}$\pm 1.57$} \\
    \bottomrule
  \end{tabular*}}
\end{table}

\paragraph{\rspadd{Replacing trained vectors.}}
At 7B, full replacement leaves LTPO's MATH-500 accuracy unchanged at $82.87\%$, but reduces TTSV from $83.60\%$ to $82.13\%$ and SoftCoT from $83.67\%$ to $82.53\%$. GSM8K changes by at most $0.08$\,pp under full replacement across the three methods at 7B. Thus, random replacement retains much of the performance at both scales, although the 7B MATH-500 results show a modest advantage for the trained vectors in two methods.

\paragraph{Disrupting learned structure.}
At 7B, dimension permutation changes MATH-500 accuracy from $82.87\%$ to $83.20\%$ for LTPO and leaves SoftCoT unchanged at $83.67\%$, while TTSV decreases from $83.60\%$ to $82.13\%$. These results show substantial retention of performance after disrupting learned structure, with method- and benchmark-dependent differences rather than uniform equivalence.

Overtraining (probe (iii)) has not been evaluated at 7B. The overtraining results in Table~\ref{tab:overtrained_probes} apply only to the 1.5B model.

%% file: appendix/law_scale_ablation.tex
\section{Sensitivity to the RSP Distribution and Scale}\label{app:law_scale}

\S\ref{sec:rsp_definition} draws RSP from an isotropic Gaussian fitted to the entrywise mean and variance of the embedding table. Proposition~\ref{thm:minimax_directional_coverage} motivates this choice, since isotropy covers every direction equally under a fixed second-moment budget. Here we test whether the choice matters in practice. We vary (a) the shape of the distribution and (b) its scale and the number of injected tokens. For (a), we compare the isotropic law with two alternatives that each change one of its properties. \textit{Diagonal}, $\mathcal{N}(\boldsymbol{\mu}_v,\mathrm{diag}(\boldsymbol{\sigma}_v^2))$ with $\boldsymbol{\mu}_v$ and $\boldsymbol{\sigma}_v$ the per-dimension mean and standard deviation of $\mathbf{W}_E$, follows the anisotropy of the embedding table. \textit{Sphere} draws a direction uniformly on the unit sphere and fixes the norm to the mean token-embedding norm, which removes the randomness in length. All runs use Qwen2.5-Math-Instruct with \textit{suffix} injection and greedy decoding, and report the mean and standard deviation over five RSP seeds.

\begin{table}[!ht]
  \centering
  \caption{\textbf{Choice of distribution.} Accuracy (\%, mean $\pm$ std, five seeds) with $L=20$ \textit{suffix} tokens. \textit{Isotropic} is the RSP law of Eq.~\eqref{eq:rsp_def}. \textit{Diagonal} and \textit{Sphere} are defined in the text above. The baseline uses no RSP and is deterministic.}
  \label{tab:law_ablation}
  \footnotesize
  \setlength{\tabcolsep}{6pt}
  \renewcommand{\arraystretch}{0.95}
  \begin{tabular}{llccc}
    \toprule
    \textbf{Model} & \textbf{Distribution} & \textbf{GSM8K} & \textbf{MATH-500} & \textbf{AIME24} \\
    \midrule
    \multirow{4}{*}{Qwen2.5-Math-1.5B-Instruct}
     & Baseline (no RSP) & $85.29$ & $73.00$ & $10.00$ \\
     & Isotropic (RSP) & $84.99$\,{\tiny\color{black!65}$\pm 0.22$} & $74.12$\,{\tiny\color{black!65}$\pm 1.36$} & $13.33$\,{\tiny\color{black!65}$\pm 5.27$} \\
     & Diagonal & $85.29$\,{\tiny\color{black!65}$\pm 0.39$} & $73.12$\,{\tiny\color{black!65}$\pm 0.69$} & $9.33$\,{\tiny\color{black!65}$\pm 2.49$} \\
     & Sphere & $84.94$\,{\tiny\color{black!65}$\pm 0.20$} & $73.56$\,{\tiny\color{black!65}$\pm 0.56$} & $12.00$\,{\tiny\color{black!65}$\pm 4.00$} \\
    \midrule
    \multirow{4}{*}{Qwen2.5-Math-7B-Instruct}
     & Baseline (no RSP) & $95.45$ & $83.20$ & $13.33$ \\
     & Isotropic (RSP) & $95.60$\,{\tiny\color{black!65}$\pm 0.08$} & $83.52$\,{\tiny\color{black!65}$\pm 0.56$} & $16.67$\,{\tiny\color{black!65}$\pm 0.00$} \\
     & Diagonal & $95.83$\,{\tiny\color{black!65}$\pm 0.17$} & $82.84$\,{\tiny\color{black!65}$\pm 0.48$} & $15.33$\,{\tiny\color{black!65}$\pm 1.64$} \\
     & Sphere & $95.65$\,{\tiny\color{black!65}$\pm 0.12$} & $83.48$\,{\tiny\color{black!65}$\pm 0.32$} & $14.67$\,{\tiny\color{black!65}$\pm 1.64$} \\
    \bottomrule
  \end{tabular}
\end{table}

\begin{table}[!ht]
  \centering
  \caption{\textbf{Scale and token count.} MATH-500 accuracy (\%, mean $\pm$ std, five seeds) on Qwen2.5-Math-1.5B-Instruct with \textit{suffix} injection. Each row sweeps one axis and keeps the others at the default ($\alpha=\beta=1$, $L=20$). $\alpha$ multiplies $\sigma_E$. $\beta$ rescales each vector to $\beta$ times the mean token-embedding norm. Each sweep was run independently, so the default setting appears with slightly different values across rows. The no-RSP baseline is $73.00$.}
  \label{tab:scale_ablation}
  \footnotesize
  \setlength{\tabcolsep}{3.5pt}
  \renewcommand{\arraystretch}{0.95}
  \begin{tabular}{lccccccc}
    \toprule
    $\sigma$ scale $\alpha$ & $0.25$ & $0.5$ & $1$ & $2$ & $4$ & & \\
    MATH-500 & $73.88$\,{\tiny\color{black!65}$\pm 1.27$} & $73.92$\,{\tiny\color{black!65}$\pm 0.84$} & $74.32$\,{\tiny\color{black!65}$\pm 1.36$} & $73.56$\,{\tiny\color{black!65}$\pm 0.68$} & $73.84$\,{\tiny\color{black!65}$\pm 0.55$} & & \\
    \midrule
    Norm scale $\beta$ & $0.25$ & $0.5$ & $1$ & $2$ & $4$ & & \\
    MATH-500 & $73.92$\,{\tiny\color{black!65}$\pm 0.52$} & $73.80$\,{\tiny\color{black!65}$\pm 0.55$} & $74.24$\,{\tiny\color{black!65}$\pm 0.79$} & $73.52$\,{\tiny\color{black!65}$\pm 0.59$} & $74.16$\,{\tiny\color{black!65}$\pm 0.52$} & & \\
    \midrule
    Tokens $L$ & $1$ & $4$ & $8$ & $16$ & $32$ & $64$ & $128$ \\
    MATH-500 & $73.72$\,{\tiny\color{black!65}$\pm 0.70$} & $74.48$\,{\tiny\color{black!65}$\pm 0.50$} & $73.12$\,{\tiny\color{black!65}$\pm 0.72$} & $74.28$\,{\tiny\color{black!65}$\pm 1.21$} & $74.48$\,{\tiny\color{black!65}$\pm 1.20$} & $72.92$\,{\tiny\color{black!65}$\pm 0.92$} & $71.08$\,{\tiny\color{black!65}$\pm 1.03$} \\
    \bottomrule
  \end{tabular}
\end{table}

\paragraph{Choice of distribution.}
Isotropy has the highest mean accuracy on MATH-500 and AIME24 at both scales, while Diagonal is highest on GSM8K (Table~\ref{tab:law_ablation}). The gaps are small. On MATH-500 the three distributions span $73.12$--$74.12\%$ at 1.5B and $82.84$--$83.52\%$ at 7B, and on GSM8K they span $84.94$--$85.29\%$ and $95.60$--$95.83\%$. Each gap is smaller than the seed standard deviation. AIME24 varies more ($9.33$--$13.33\%$ at 1.5B), but its 30 problems move accuracy in steps of $3.33$ points and the standard deviations reach $5.27$ points. The table therefore shows at most a slight edge for isotropy and no evidence that another distribution would be preferable. This agrees with \S\ref{sec:content_probes}, where the effect depends on injecting vectors of a suitable magnitude and not on their precise content.

\paragraph{Scale and token count.}
Scaling $\sigma_E$ or the norm by any factor between $0.25$ and $4$, a sixteen-fold range, keeps MATH-500 accuracy within $73.52$--$74.32\%$, again inside one standard deviation (Table~\ref{tab:scale_ablation}). The token count behaves the same way up to $L=32$ ($73.12$--$74.48\%$). Beyond that, accuracy declines to $72.92\%$ at $L=64$ and $71.08\%$ at $L=128$, which falls below the no-RSP baseline of $73.00\%$. A long random prompt therefore starts to hurt. This is consistent with Theorem~\ref{thm:decay}, where a larger $L$ raises the attention mass on random tokens and delays the dilution that lets the response commit to a single completion.

%% file: appendix/temperature_comparison.tex
\section{Comparison with Intermediate and Annealed Temperatures}\label{app:temperature}

Figure~\ref{fig:pass_at_n} compares RSP only with temperature sampling at $\tau{=}1.0$. Table~\ref{tab:rsp_vs_temp} shows that raising the temperature to $\tau{=}3.0$, which matches RSP's distributional shift in magnitude, collapses accuracy. That value was chosen to match magnitude, not to maximize Pass@$N$. Here we extend Figure~\ref{fig:pass_at_n} to milder and time-varying temperatures and ask whether either can reach RSP's outcome diversity. We compare (a) constant temperatures $\tau \in \{1.0, 1.2, 1.5, 1.8, 2.0\}$ and (b) annealed schedules that start at $\tau_s \in \{2.0, 2.5, 3.0\}$ and decrease to $1.0$, either linearly or along a cosine curve. A schedule gives high diversity at the first tokens and then lets the response settle, which mimics the early-then-fading influence of RSP. The annealing length $T$ is the mean generation length at $\tau{=}1.0$ ($547$ tokens on MATH-500 and $913$ on AIME24), so the schedule reaches $1.0$ within an average completion. We also include (c) the trained soft prompts as a reference. All runs follow the setting of Figure~\ref{fig:pass_at_n} (Qwen2.5-Math-1.5B-Instruct, 16 samples per problem, top-$p{=}0.95$).

\begin{table}[!ht]
  \centering
  \caption{\textbf{Pass@$k$ under intermediate and annealed temperatures.} Qwen2.5-Math-1.5B-Instruct, 16 samples per problem. Annealed schedules start at $\tau_s$ and reach $\tau{=}1.0$ after $T$ tokens. Soft prompts and RSP are sampled at $\tau{=}1.0$. Bold marks the highest value in each column.}
  \label{tab:temperature_passk}
  \footnotesize
  \setlength{\tabcolsep}{5pt}
  \renewcommand{\arraystretch}{0.95}
  \begin{tabular}{llcccccc}
    \toprule
    & & \multicolumn{3}{c}{\textbf{MATH-500}} & \multicolumn{3}{c}{\textbf{AIME24}} \\
    \cmidrule(lr){3-5}\cmidrule(lr){6-8}
    \textbf{Category} & \textbf{Setting} & $k{=}4$ & $k{=}8$ & $k{=}16$ & $k{=}4$ & $k{=}8$ & $k{=}16$ \\
    \midrule
    \multirow{5}{*}{(a) Constant $\tau$}
     & $\tau{=}1.0$ & $85.09$ & $88.80$ & $91.80$ & $17.63$ & $19.21$ & $20.00$ \\
     & $\tau{=}1.2$ & $84.92$ & $88.85$ & $91.60$ & $15.21$ & $19.08$ & $23.33$ \\
     & $\tau{=}1.5$ & $82.32$ & $87.09$ & $90.40$ & $13.63$ & $18.32$ & $23.33$ \\
     & $\tau{=}1.8$ & $66.44$ & $75.25$ & $82.40$ & $4.45$ & $5.88$ & $6.67$ \\
     & $\tau{=}2.0$ & $45.86$ & $58.22$ & $68.40$ & $1.67$ & $3.33$ & $6.67$ \\
    \midrule
    \multirow{3}{*}{(b) Linear anneal}
     & $\tau_s{=}2.0$ & $49.09$ & $63.04$ & $73.60$ & $0.00$ & $0.00$ & $0.00$ \\
     & $\tau_s{=}2.5$ & $20.16$ & $30.98$ & $43.80$ & $0.83$ & $1.67$ & $3.33$ \\
     & $\tau_s{=}3.0$ & $15.68$ & $25.12$ & $36.20$ & $0.00$ & $0.00$ & $0.00$ \\
    \midrule
    \multirow{3}{*}{(b) Cosine anneal}
     & $\tau_s{=}2.0$ & $44.59$ & $59.23$ & $72.00$ & $0.83$ & $1.67$ & $3.33$ \\
     & $\tau_s{=}2.5$ & $20.07$ & $31.18$ & $43.40$ & $0.83$ & $1.67$ & $3.33$ \\
     & $\tau_s{=}3.0$ & $17.61$ & $28.09$ & $40.80$ & $0.00$ & $0.00$ & $0.00$ \\
    \midrule
    \multirow{3}{*}{(c) Soft prompts}
     & LTPO & $\mathbf{85.73}$ & $89.33$ & $91.60$ & $15.51$ & $19.17$ & $23.33$ \\
     & SoftCoT & $84.59$ & $88.79$ & $91.40$ & $18.20$ & $23.00$ & $30.00$ \\
     & TTSV & $84.59$ & $88.43$ & $91.40$ & $15.43$ & $17.96$ & $20.00$ \\
    \midrule
    \multicolumn{2}{l}{RSP (\textit{suffix}, independent draws)} & $85.58$ & $\mathbf{89.58}$ & $\mathbf{92.80}$ & $\mathbf{21.36}$ & $\mathbf{27.67}$ & $\mathbf{36.67}$ \\
    \bottomrule
  \end{tabular}
\end{table}

\paragraph{Constant temperature.}
Raising the temperature does not raise Pass@$N$ (Table~\ref{tab:temperature_passk}). On MATH-500, Pass@16 is highest at $\tau{=}1.0$ ($91.80\%$) and decreases as $\tau$ grows. On AIME24, $\tau{=}1.2$ and $1.5$ lift Pass@16 from $20.00\%$ to $23.33\%$, one additional problem out of 30. From $\tau{=}1.8$ onward, accuracy collapses on both benchmarks.

\paragraph{Annealed temperature.}
Annealing does not help either. Every schedule falls below constant $\tau{=}1.0$, and on AIME24 no schedule exceeds $3.33\%$ at $k{=}16$. Even the mildest schedule (linear, $\tau_s{=}2.0$) reaches only $73.60\%$ on MATH-500 at $k{=}16$, below constant $\tau{=}1.8$ ($82.40\%$) although it is cooler for all but the first ${\sim}110$ tokens. This suggests that the high-temperature tokens at the start already derail the solution and later cooling cannot recover it. The pattern agrees with our view that the first decoding steps decide the trajectory. RSP acts on the same early steps but keeps single-sample accuracy, because it reranks tokens instead of flattening the whole distribution (Table~\ref{tab:rsp_vs_temp}).

\paragraph{Comparison with RSP.}
RSP gives the highest Pass@$N$ on AIME24 at every $k$, reaching $36.67\%$ at $k{=}16$ against $23.33\%$ for the best temperature setting and $30.00\%$ for the best soft prompt. On MATH-500, all non-collapsed methods sit near the ceiling ($90.4$--$92.8\%$ at $k{=}16$). RSP is highest at $k{=}8$ and $k{=}16$, while LTPO is slightly higher at $k{=}4$. Temperature only flattens the distribution and preserves the token ranking (Table~\ref{tab:rsp_vs_temp}), whereas RSP changes the ranking. Neither a milder nor a time-varying temperature closes this gap.

%% file: appendix/dapo.tex
\section{Additional DAPO Experiment}\label{app:dapo}

\rspadd{We additionally explore whether RSP can be combined with DAPO training~\citep{yu2025dapo}, a variant of GRPO~\citep{shao2024deepseekmath}. RSP is applied during rollout collection, while evaluation uses no RSP. Each condition uses a single training run.}

\subsection{From GRPO to DAPO}\label{app:grpo}

The vanilla GRPO objective \citep{shao2024deepseekmath} for a group of $G$ rollouts $\{y_i\}_{i=1}^G$ (each $y_i = (y_{i,1},\dots,y_{i,|y_i|})$ a token sequence, and we use $y$ rather than $o$ throughout this appendix to avoid collision with the head output $\mathbf{o}_i^{(\ell)}$ of \S\ref{sec:improve}) with rewards $\{R_i\}_{i=1}^G$ is
\begin{equation}\label{eq:grpo}
\mathcal{J}_{\text{GRPO}}(\theta) = \mathbb{E}\!\left[\frac{1}{G}\sum_{i=1}^{G}\frac{1}{|y_i|}\sum_{t=1}^{|y_i|}\!\left(\min\!\big(r_{i,t}\hat{A}_{i,t},\,\text{clip}(r_{i,t},1{-}\varepsilon,1{+}\varepsilon)\hat{A}_{i,t}\big) - \beta\,D_{\text{KL}}(\pi_\theta\|\pi_{\text{ref}})\right)\right],
\end{equation}
with importance ratio $r_{i,t} = \pi_\theta(y_{i,t}\mid q,y_{i,<t})/\pi_{\theta_{\text{old}}}(y_{i,t}\mid q,y_{i,<t})$, group-relative advantage $\hat{A}_{i,t} = (R_i - \mathrm{mean}(\{R_i\}))/\mathrm{std}(\{R_i\})$, clipping range $\varepsilon$, and KL penalty $\beta\,D_{\text{KL}}$ against the reference $\pi_{\text{ref}}$.

We build on the SimpleRL-Zoo \citep{zeng2025simplerl} training pipeline and implement DAPO on top of VeRL \citep{sheng2024hybridflow}. DAPO modifies Eq.~\eqref{eq:grpo} in four ways: (i) token-level loss aggregation $1/\sum_i|y_i|$ instead of $1/|y_i|$, (ii) asymmetric clipping with $(\varepsilon_{\text{low}},\varepsilon_{\text{high}})=(0.2,0.28)$, (iii) an implementation-side negative-advantage dual-clipping safeguard $\max(L_{\text{clipped}},c\hat{A})$ with $c=10$ (active only when $\hat{A}<0$ and omitted from Eq.~\eqref{eq:dapo} below for readability), and (iv) KL terms removed. The resulting DAPO objective is
\begin{equation}\label{eq:dapo}
\mathcal{J}_{\text{DAPO}}(\theta)
\;=\;
\mathbb{E}\!\left[
\frac{1}{\sum_i|y_i|}
\sum_{i,t}
\min\!\Big(r_{i,t}\hat{A}_{i,t},\,
\text{clip}\big(r_{i,t},\,1{-}\varepsilon_{\text{low}},\,1{+}\varepsilon_{\text{high}}\big)\hat{A}_{i,t}\Big)
\right].
\end{equation}

\subsection{Training Setup and RSP Implementation}\label{app:dapo_rsp_impl}

\rspadd{We implement DAPO on top of SimpleRL-Zoo~\citep{zeng2025simplerl} and train Qwen2.5-Math-7B on a MATH Level 3--5 subset of approximately 8{,}500 prompts for 100 steps. Each step uses $G=8$ rollouts per prompt at $\tau=1.0$, a training batch size of $1{,}024$, and four PPO mini-batches of $256$. We compare DAPO with DAPO~+~RSP, which uses independently sampled \textit{suffix} RSPs with $L=20$.}

For DAPO~+~RSP, each rollout receives $\mathbf{H}_i \in \mathbb{R}^{L \times d}$. Each $\mathbf{H}_i$ is generated deterministically from a per-step rollout seed and injected via a forward hook. It is not a trainable parameter, so the learning signal is the policy adapting to arbitrary $\mathbf{H}_i$ rather than a learned prompt. Including $\mathbf{H}_i$ in the log-probabilities at both rollout and update time keeps the update on-policy with respect to the RSP-conditioned distribution (avoiding off-policy mismatch).

\begin{table}[H]
\centering
\caption{DAPO~/~DAPO~+~RSP training setup on Qwen2.5-Math-7B.}
\label{tab:dapo_hp}
\footnotesize
\setlength{\tabcolsep}{4pt}
\renewcommand{\arraystretch}{0.9}
\scalebox{0.8594800}{%
\begin{tabular}{ll}
\toprule
\textbf{Field} & \textbf{Value} \\
\midrule
Train data & MATH Level 3--5 subset (\texttt{simplelr\_math\_35}, \textasciitilde 8{,}500 prompts) \\
Prompt template & \texttt{qwen-boxed} \\
Train batch size & 1024 \\
PPO mini-batch size & 256 (4 PPO updates per rollout) \\
Rollouts per prompt $G$ & 8 \\
Max prompt / response length & 2{,}028 / 2{,}048 \\
Rollout sampling & temperature $1.0$, top-$p$ $1.0$, top-$k$ $50$ \\
\midrule
$(\varepsilon_{\text{low}},\varepsilon_{\text{high}})$ & $(0.2,\,0.28)$ \\
Dual clip $c$ & $10.0$ \\
Loss aggregation & token-mean ($1/\sum_i|y_i|$) \\
KL terms & disabled \\
Entropy coefficient & $0$ \\
\midrule
Optimizer & AdamW, learning rate $5\times10^{-7}$ (constant), no warmup \\
Total training & 100 optimization steps \\
Save / eval frequency & every 10 steps \\
\midrule
RSP (DAPO + RSP only) & \textit{suffix} injection, $L=20$, freshly drawn per rollout \\
\midrule
Hardware & $4 \times$ NVIDIA B200 GPUs \\
Wall-clock training time & $\sim$$12$ hours per run \\
\bottomrule
\end{tabular}}
\end{table}

\subsection{Evaluation Protocol}

\rspadd{We evaluate on GSM8K~\citep{cobbe2021training}, MATH-500~\citep{lightman2024lets}, College Math, Minerva Math~\citep{lewkowycz2022minerva}, and AIME24. AIME24 uses Avg@32 at $\tau=1.0$, while the other four benchmarks use greedy decoding.}

Each saved checkpoint is converted to HuggingFace format and evaluated using the SimpleRL-Zoo \texttt{eval\_math\_nodes.sh} pipeline, which selects sampling parameters per benchmark to reduce variance on the smaller competition sets:

\begin{table}[H]
\centering
\caption{Per-benchmark evaluation protocol used by SimpleRL-Zoo \texttt{sh/eval.sh}.}
\label{tab:dapo_eval_protocol}
\footnotesize
\setlength{\tabcolsep}{6pt}
\renewcommand{\arraystretch}{0.9}
\scalebox{0.8594800}{%
\begin{tabular}{lcccl}
\toprule
\textbf{Benchmark} & \textbf{\# Problems} & \textbf{Temperature} & \textbf{$n_{\text{sampling}}$} & \textbf{Metric} \\
\midrule
AIME 2024       & 30    & 1.0 & 32 & Avg@32 \\
AMC 2023        & 40    & 1.0 & 32 & Avg@32 \\
MATH-500        & 500   & 0.0 & 1  & accuracy (mean@1) \\
GSM8K           & 1{,}319 & 0.0 & 1  & accuracy \\
OlympiadBench   & 675   & 0.0 & 1  & accuracy \\
College Math    & 2{,}818 & 0.0 & 1  & accuracy \\
Minerva Math    & 272   & 0.0 & 1  & accuracy \\
GaoKao 2023-EN  & 385   & 0.0 & 1  & accuracy \\
\bottomrule
\end{tabular}}
\end{table}

The maximum number of generated tokens is $16{,}000$ across all benchmarks. The reported average is the unweighted mean over the five benchmarks (GSM8K, MATH-500, College Math, Minerva Math, AIME 2024).

\subsection{Results}

\begin{table}[!ht]
\centering
\caption{Per-benchmark accuracy (\%) at each method's peak step on Qwen2.5-Math-7B. \textbf{Bold} denotes the better of the two RL methods on each benchmark. Avg is the unweighted five-benchmark mean.}
\label{tab:dapo_peak}
\footnotesize
\setlength{\tabcolsep}{6pt}
\renewcommand{\arraystretch}{0.9}
\scalebox{0.8594800}{%
\begin{tabular}{l ccccc c}
\toprule
\textbf{Method} & \textbf{GSM8K} & \textbf{MATH-500} & \textbf{College} & \textbf{Minerva} & \textbf{AIME24} & \textbf{Avg} \\
\midrule
Base                & 57.2 & 52.6 & 18.3 & 13.6 & 8.6  & 30.06 \\
DAPO                & 86.3 & 78.2 & 41.4 & 37.5 & 22.6 & 53.20 \\
\rowcolor{blue!10}
DAPO~+~RSP          & \textbf{87.7} & \textbf{78.6} & \textbf{42.2} & \textbf{39.7} & \textbf{23.6} & \textbf{54.36} \\
\bottomrule
\end{tabular}}
\end{table}

\begin{figure}[htbp]
  \centering
  \includegraphics[width=0.55\textwidth]{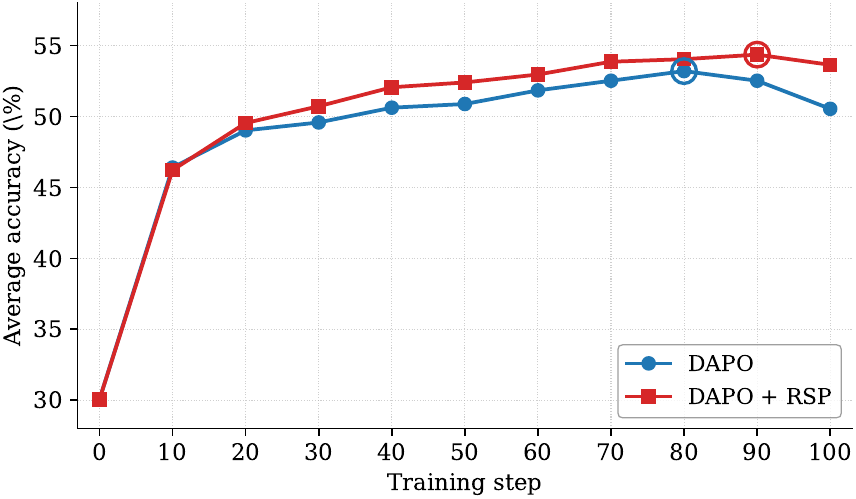}
  \caption{\rspadd{\textbf{DAPO training trajectories.} Five-benchmark average accuracy on Qwen2.5-Math-7B. Each condition uses a single training run.}}
  \label{fig:dapo_curve}
\end{figure}

\rspadd{Figure~\ref{fig:dapo_curve} and Table~\ref{tab:dapo_peak} summarize the observed training trajectories and each condition's peak checkpoint. DAPO~+~RSP reaches $54.36\%$ at step 90, compared with DAPO's peak of $53.20\%$ at step 80, a difference of $1.16$\,pp. At step 100, the difference is $3.10$\,pp. The curves cross around step 20.}

\begin{table}[H]
\centering
\caption{Five-benchmark average accuracy (\%) at every checkpoint, computed over GSM8K, MATH-500, College Math, Minerva Math, and AIME 2024. \textbf{Bold} marks each method's peak.}
\label{tab:dapo_per_step}
\footnotesize
\setlength{\tabcolsep}{6pt}
\renewcommand{\arraystretch}{0.9}
\scalebox{0.8594800}{%
\begin{tabular}{lccccccccccc}
\toprule
\textbf{Step} & \textbf{0} & \textbf{10} & \textbf{20} & \textbf{30} & \textbf{40} & \textbf{50} & \textbf{60} & \textbf{70} & \textbf{80} & \textbf{90} & \textbf{100} \\
\midrule
DAPO & 30.06 & 46.40 & 49.02 & 49.58 & 50.62 & 50.88 & 51.84 & 52.52 & \textbf{53.20} & 52.52 & 50.54 \\
DAPO~+~RSP & 30.06 & 46.22 & 49.54 & 50.72 & 52.06 & 52.40 & 52.96 & 53.86 & 54.04 & \textbf{54.36} & 53.64 \\
\bottomrule
\end{tabular}}
\end{table}